\documentclass{article}

\usepackage[nonatbib,final]{nips_2017}
\usepackage[square,numbers]{natbib}
\usepackage[hidelinks,colorlinks]{hyperref}
\usepackage{titlesec}
\titlespacing*{\subsection}{0 pt}{3.0pt plus 0pt minus 2.4pt}{3.0pt plus 0pt minus 2.4pt}
\usepackage[utf8]{inputenc} 
\usepackage[T1]{fontenc}    
\usepackage{hyperref}       
\usepackage{url}            
\usepackage{booktabs}       
\usepackage{amsfonts}       
\usepackage{nicefrac}       
\usepackage{microtype}      
\usepackage{algorithm}
\usepackage{algorithmic}

\usepackage{amsmath}
\usepackage{amsthm}
\usepackage{amssymb}
\usepackage{mathrsfs}
\usepackage[outdir=./]{epstopdf}
\usepackage{caption}
\usepackage[framemethod=TikZ]{mdframed}
\usepackage{lipsum}
\usepackage{subcaption}
\usepackage{multirow}
\usepackage{comment}
\usepackage{framed}
\usepackage{color}
\usepackage{thmtools,thm-restate}
\usepackage{booktabs}
\newtheorem{thm}{Theorem}
\newtheorem{thm2}{Theorem}
\newtheorem{cor}{Corollary}
\newtheorem{lem}{Lemma}

\newcommand*{\Prob}{\mathbb{P}}

\newcommand*{\window}{t}
\newcommand*{\Real}{\mathbb{R}}
\newcommand{\largenorm}[1]{{\Big\| #1 \Big\|_{2}}}
\newcommand{\largeabsnorm}[1]{{\Big\| #1 \Big\|_{1}}}
\newcommand*{\absnorm}[1]{\| #1 \|_1}
\newcommand*{\poly}[1]{\text{poly}(#1)}
\newcommand*{\Oh}[1]{\mathcal{O}(#1)}
\newcommand{\polylog}[1]{\text{polylog}(#1)}

\newcommand{\norm}[1]{{\| #1 \|_{2}}}

\newcommand*{\E}{\mathbb{E}}

\title{Learning Overcomplete HMMs}

\author{
	Vatsal Sharan\\
	Stanford University\\
	\texttt{vsharan@stanford.edu} \\
	\And
	Sham Kakade \\
	University of Washington \\
	\texttt{sham@cs.washington.edu} \\
	\AND
	Percy Liang \\
	Stanford University \\
	\texttt{pliang@cs.stanford.edu} \\
	\And
	Gregory Valiant \\
	Stanford University \\
	\texttt{valiant@stanford.edu} \\
}

\begin{document}
	\definecolor{mydarkblue}{rgb}{0,0.08,0.45}
	\hypersetup{
		colorlinks=true,
		linkcolor=mydarkblue,
		citecolor=mydarkblue,
		filecolor=mydarkblue,
		urlcolor=mydarkblue,
	}
\maketitle
\vspace{-10pt}
\vspace{-12pt}
\begin{abstract}
\vspace{-8pt}
We study the problem of learning overcomplete HMMs---those that have many hidden states but a small output alphabet. Despite having significant practical importance, such HMMs are poorly understood with no known positive or negative results for efficient learning. In this paper, we present several new results---both positive and negative---which help define the boundaries between the tractable and intractable settings. Specifically, we show positive results for a large subclass of HMMs whose transition matrices are sparse, well-conditioned, and have small probability mass on short cycles.  On the other hand, we show that learning is impossible given only a polynomial number of samples for HMMs with a small output alphabet and whose transition matrices are random regular graphs with large degree. We also discuss these results in the context of learning HMMs which can capture long-term dependencies.
\end{abstract}

\vspace{-18pt}
\section{Introduction}
\vspace{-5pt}
Hidden Markov Models (HMMs) are commonly used for data with natural sequential structure (e.g., speech, language, video).
This paper focuses on \emph{overcomplete} HMMs, where the number of output symbols $m$ is much smaller than the number of hidden states $n$.
As an example, for an HMM that outputs natural language documents one character at a time, the number of characters $m$ is quite small, but the number of hidden states $n$ would need to be very large to encode the rich syntactic, semantic, and discourse structure of the document.


Most algorithms for learning HMMs with provable guarantees
assume the transition $T \in \mathbb R^{n \times n}$ and observation $O \in \mathbb R^{m \times n}$ matrices are full rank \cite{anandkumar2012method,anandkumar13tensor,mossel2005learning} and hence do not apply to the overcomplete regime. A notable exception is the recent work of  \citet{huang14hmm} who studied this setting where $m \ll n$ and showed that \emph{generic} HMMs can be learned in polynomial time given \emph{exact} moments of the output process (which requires infinite data). Though understanding properties of generic HMMs is an important first step, in reality, HMMs with a large number of hidden states typically have structured, non-generic transition matrices---e.g., consider sparse transition matrices or transition matrices of factorial HMMs \cite{ghahramani1997factorial}. \citet{huang14hmm} also assume access to exact moments, which leaves open the question of when learning is possible with efficient sample complexity.
Summarizing, we are interested in the following questions:
\begin{enumerate}
	\item What are the fundamental limitations for learning overcomplete HMMs?
	\item What properties of HMMs make learning possible with polynomial samples?
	\item Are there structured HMMs which can be learned in the overcomplete regime?
\end{enumerate}


\vspace{-5pt}
\noindent \textbf{Our contributions.} We make progress on all three questions in this work, sharpening our understanding of the boundary between tractable and intractable learning. We begin by stating a negative result, which perhaps explains some of the difficulty of obtaining strong learning guarantees in the overcomplete setting.

\begin{thm}\label{thm:lower_bound}
	The parameters of HMMs where i) the transition matrix encodes a random walk on a regular graph on $n$ nodes with degree polynomial in $n$, ii) the output alphabet $m=\polylog{n}$ and, iii) the output distribution for each hidden state is chosen uniformly and independently at random, cannot be learned (even approximately) using polynomially many samples over any window length polynomial in $n$, with high probability over the choice of the observation matrix.
\end{thm}
\vspace{-6pt}
Theorem \ref{thm:lower_bound} is somewhat surprising, as parameters of HMMs with such transition matrices can be easily learned in the non-overcomplete ($m \ge n$) regime. This is because such transition matrices are full-rank and their condition numbers are polynomial in $n$; hence spectral techniques such as \citet{anandkumar13tensor} can be applied. Theorem~\ref{thm:lower_bound} is also fundamentally of a different nature as compared to lower bounds based on parity with noise reductions for HMMs \cite{mossel2005learning}, as ours is information-theoretic.\footnote{Parity with noise is information theoretically easy given observations over a window of length at least the number of inputs to the parity. This is linear in the number of hidden states of the parity with noise HMM, whereas Theorem \ref{thm:lower_bound} says that the sample complexity must be super polynomial for any polynomial sized window.} Also, it seems far more damning as the hard cases are seemingly innocuous classes such as random walks on dense graphs. The lower bound also shows that analyzing generic or random HMMs might not be the right framework to consider in the overcomplete regime as these might not be learnable with polynomial samples even though they are identifiable. This further motivates the need for understanding HMMs with structured transition matrices. We provide a proof of Theorem \ref{thm:lower_bound} with more explicitly stated conditions in Appendix \ref{sec:lower_bound_proof}.

\vspace{-2pt}
For our positive results we focus on understanding properties of structured transition matrices which make learning tractable. To disentangle additional complications due to the choice of the observation matrix, we will assume that the observation matrix is drawn at random throughout the paper. Long-standing open problems on learning \emph{aliased} HMMs (HMMs where multiple hidden states have identical output distributions) \cite{blackwell57identifiable,ito1992identifiability,weiss2015learning} hint that understanding learnability with respect to properties of the observation matrix is a daunting task in itself, and is perhaps best studied separately from understanding how properties of the transition matrix affect learning.

Our positive result on learnability (Theorem \ref{thm:learnability}) depends on two natural graph-theoretic properties of the transition matrix. We consider transition matrices which are i) {sparse} (hidden states have constant degree) and ii) have small probability mass on cycles shorter than $10\log_m n$ states---and show that these HMMs can be learned efficiently using tensor decomposition and the method of moments, given random observation matrices. The condition prohibiting short cycles might seem mysterious. Intuitively, we need this condition to ensure that the Markov Chain visits a sufficient large portion of the state space in a short interval of time, and in fact the condition stems from information-theoretic considerations. We discuss these further in Sections~\ref{sec:examples} and \ref{subsec:assumptions}. We also discuss how our results relate to learning HMMs which capture long-term dependencies in their outputs, and introduce a new notion of how well an HMM captures long-term dependencies. These are discussed in Section \ref{sec:long_term}. 


We also show new identifiability results for sparse HMMs. These results provide a finer picture of identifiability than \citet{huang14hmm}, as ours hold for sparse transition matrices which are not generic.

\noindent \textbf{Technical contribution.} To prove Theorem \ref{thm:learnability} we show that the Khatri-Rao product of dependent random vectors is well-conditioned under certain conditions. Previously, \citet{bhaskara2014smoothed} showed that the Khatri-Rao product of independent random vectors is well-conditioned to perform a smoothed analysis of tensor decomposition, their techniques however do not extend to the dependent case. For the dependent case, we show a similar result using a novel Markov chain coupling based argument which relates the condition number to the best coupling of output distributions of two random walks with disjoint starting distributions. The technique is outlined in Section \ref{sec:technical_contri}.

\noindent \textbf{Related work.} Spectral methods for learning HMMs have been studied in \citet{anandkumar13tensor,bhaskara13tensor,allman09identifiability,hsu2012spectral}, but these results require $m\ge n$. In \citet{allman09identifiability}, the authors show that that HMMs are identifiable given moments of continuous observations over a time interval of length $N=2\tau+1$ for some $\tau$ such that ${\tau+m-1 \choose m-1}\ge n$. When $m\ll n$ this requires $\tau=\Oh{n^{1/m}}$. \citet{bhaskara13tensor} give another bound on window size which requires $\tau=\Oh{n/m}$. However, with a output alphabet of size $m$, specifying all moments in a $N$ length continuous time interval requires $m^N$ time and samples, and therefore all of these approaches lead to exponential runtimes when $m$ is constant with respect to $n$. Also relevant is the work by \citet{anandkumar2015learning} on guarantees for learning certain latent variable models such as Gaussian mixtures in the overcomplete setting through tensor decomposition. As mentioned earlier, the work closest to ours is \citet{huang14hmm} who showed that generic HMMs are identifiable with $\tau=\Oh{\log_m n}$, which gives the first polynomial runtimes for the case when $m$ is constant.

\noindent \textbf{Outline.} Section \ref{sec:setup} introduces the notation and setup. It also provides examples and a high-level overview of our proof approach. Section \ref{sec:learnability} states the learnability result, discusses our assumptions and HMMs which satisfy these assumptions. Section \ref{sec:identifiability} contains our identifiability results for sparse HMMs. Section \ref{sec:long_term} discusses natural measures of long-term dependencies in HMMs. We state the lower bound for learning dense, random HMMs in Section \ref{sec:lower_bound}. We conclude in Section \ref{sec:conclusion}. We provide proof sketches in the main body, rigorous proofs are deferred to the Appendix.  


\vspace{-8pt}
\section{Setup and overview}\label{sec:setup}
\vspace{-8pt}
In this section we first introduce the required notation, and then outline the method of moments approach for parameter recovery. We also go over some examples to provide a better understanding of the classes of HMMs we aim to learn, and give a high level proof strategy.

\subsection{Notation and preliminaries}

We will denote the output at time $t$ by $y_t$ and the hidden state at time $t$ by $h_t$. Let the number of hidden states be $n$ and the number of observations be $m$. Assume that the output alphabet is $\{0,\dots,m-1\}$ without loss of generality. Let $T$ be the transition matrix and $O$ be the observation matrix of the HMM, both of these are defined so that the columns add up to one. For any matrix $A$, we refer to the $i$th column of $A$ as $A_i$. $T'$ is defined as the transition matrix of the time-reversed Markov chain, but we do not assume reversibility and hence $T$ may not equal $T'$.
Let $y_{i}^{j} = y_{i}, \dots, y_j$ denote the sequence of outputs from time $i$ to time $j$. Let $l_{i}^{j}= l_{i}, \dots, l_j$ refer to a string of length $i+j-1$ over the output alphabet, denoting a particular output sequence from time $i$ to $j$.
Define a bijective mapping $L$ which maps an output sequence $l_1^{\tau}\in\{0,\dots, m-1\}^{\tau}$ into an index $L(l_1^{\tau})\in \{1, \dots, m^\tau\}$ and the associated inverse mapping $L^{-1}$. 

Throughout the paper, we assume that the transition matrix $T$ is ergodic, and hence has a stationary distribution. We also assume that every hidden state has stationary probability at least $1/\poly{n}$. This is a necessary condition, as otherwise we might not even visit all states in $\poly{n}$ samples. We also assume that the output process of the HMM is stationary. A stochastic process is stationary if the distribution of any subset of random variables is invariant with respect to shifts in the time index---that is, $\Prob[y_{-\tau}^{\tau}=l_{-\tau}^{\tau}]=\Prob[y_{-\tau+T}^{\tau+T}=l_{-\tau}^{\tau}]$ for any $\tau,T$ and string $l_{-\tau}^{\tau}$. This is true if the initial hidden state is chosen according to the stationary distribution.

Our results depend on the conditioning of the matrix $T$ with respect to the $\ell_1$ norm. We define $\sigma_{\min}^{(1)}(T)$ as the minimum $\ell_1$ gain of the transition matrix $T$ over all vectors $x$ having unit $\ell_1$ norm (not just non-negative vectors $x$, for which the ratio would always be 1):
\begin{align}
\sigma_{\min}^{(1)}(T) = \min_{x\in \mathbb{R}^n}\frac{\absnorm{Tx}}{\absnorm{x}}\nonumber
\end{align}
$\sigma_{\min}^{(1)}(T)$ is also a natural parameter to measure the long-term dependence of the HMM---if $\sigma_{\min}^{(1)}(T)$ is large then $T$ preserves significant information about the distribution of hidden states at time 0 at a future time $t$, for all initial distributions at time 0. We discuss this further in Section \ref{sec:long_term}. 

\subsection{Tensor basics}\label{subsec:tensor}

Given a 3rd order rank-$k$ tensor $M\in \Real^{d_1\times d_2 \times d_3}$, it can be written in terms of its factor matrices $A,B$ and $C$:
\begin{align}
M = \sum_{i\in [k]}^{} A_i \otimes B_i \otimes C_i  \nonumber
\end{align}
where $A_i$ denotes the $i$th column of a matrix $A$. Here $\otimes$ denotes the tensor product: if $a,b,c \in \mathbb{R}^d$ then $a \otimes b \otimes c \in \mathbb{R}^{d\times d \times d}$ and $(a \otimes b \otimes c)_{ijk}=a_i b_j c_k$. We refer to different dimensions of a tensor as the \emph{modes} of the tensor.

We denote $M_{(k)}$ as the mode $k$ matricization of the tensor, which is the flattening of the tensor along the $k$th direction obtained by stacking all the matrix slices together. For example $T_{(1)}$ denotes flattening of a tensor $T\in \mathbb{R}^{d_1 \times d_2 \times D_3}$ to a $(d_1 \times d_2d_3)$ matrix. Recall that we denote the Khatri-Rao product of two matrices $A$ and $B$ as $(A \odot B )_i = (A_i \otimes B_i)_{(1)}$, where $(A_i \otimes B_i)_{(1)}$ denotes the flattening of the matrix $A_i \otimes B_i$ into a row vector. We denote the set $\{1,2,\dotsb,k\}=[k]$. 

Kruskal's condition \cite{kruskal1977three} says that if $A$ and $B$ are full rank and no two rows of $C$ are linearly dependent, then $M$ can be efficiently decomposed into the factors $A,B,C$ and the decomposition is unique upto scaling and permutation. The simultaneous decomposition algorithm \cite{chang1996full,leurgans1993decomposition} (Algorithm \ref{alg:learn_tensor}), is a well known algorithm to decompose tensors which satisfy Kruskal's condition.

\subsection{Method of moments for learning HMMs}\label{sec:technical_contri}
Our algorithm for learning HMMs follows the method of moments based approach, outlined for example in \citet{anandkumar2012method} and \citet{huang14hmm}. In contrast to the more popular Expectation-Maximization (EM) approach which can suffer from slow convergence and local optima \cite{redner1984mixture}, the method of moments approach ensures guaranteed recovery of the parameters under mild conditions.

The method of moments approach to learning HMMs has two high-level steps. In the first step, we write down a tensor of empirical moments of the data, such that the factors of the tensor correspond to parameters of the underlying model. In the second step, we perform tensor decomposition to recover the factors of the tensor---and then recover the parameters of the model from the factors. The key fact that enables the second step is that tensors have a unique decomposition under mild conditions on their factors, for example tensors have a unique decomposition if all the factors are full rank. The uniqueness of tensor decomposition permits unique recovery of the parameters of the model. 

We will learn the HMM using the moments of observation sequences $y_{-\tau}^{\tau}$ from time $-\tau$ to $\tau$. Since the output process is assumed to be stationary, the distribution of outputs is the same for any contiguous time interval of the same length, and we use the interval $-\tau$ to $\tau$ in our setup for convenience.
We call the length of the observation sequences used for learning the \emph{window length} $N = 2\tau+1$.
Since the number of samples required to estimate moments over a window of length $N$ is $m^N$, it is desirable to keep $N$ small.
Note that to ensure polynomial runtime and sample complexity for the method of moments approach, the window length $N$ must be $\Oh{\log_m n}$. 

We will now define our moment tensor. Given moments over a window of length $N=2\tau+1$, we can construct the third-order moment tensor $M \in \Real^{m^{\tau}\times m^{\tau} \times m}$ using the mapping $L$ from strings of outputs to indices in the tensor:
\vspace{-4pt}
\begin{align}
M_{(L(l_1^{\tau}),L(l_{-1}^{-\tau}), l_0)} = \Prob[y_{-\tau}^{\tau}=l_{-\tau}^{\tau}]. \nonumber
\end{align}
$M$ is simply the tensor of the moments of the HMM over a window length $N$, and can be estimated directly from data. We can write $M$ as an outer product because of the Markov property:
\vspace{-4pt}
\begin{align}
M = A \otimes B \otimes C \nonumber
\end{align}
where $A\in \Real^{m^\tau\times n}, B \in \Real^{m^\tau\times n}, C \in \Real^{m\times n}$ are defined as follows (here $h_0$ denotes the hidden state at time 0):
\vspace{-4pt}
\begin{align}
A_{L(l_1^{\tau}),i} &= \Prob[y_1^{\tau} = l_1^{\tau} \mid h_0 = i]\nonumber\\
B_{L(l_{-1}^{-\tau}),i} &= \Prob[y_{-1}^{-\tau} = l_{-1}^{-\tau} \mid h_0 = i]\nonumber\\
C_{l_0,i} &= \Prob[y_0 = l, h_0 = i]\nonumber
\end{align}
$T$ and $O$ can be related in a simple manner to $A, B$ and $C$.
If we can decompose the tensor $M$ into the factors $A$, $B$ and $C$, we can recover $T$ and $O$ from $A$, $B$ and $C$. We refer the reader to Algorithm \ref{alg:learn_tensor} for more details.
\begin{algorithm*}
	\caption{Learning HMMs with $m \ll n$ \cite{huang14hmm}}
	\label{alg:learn_tensor}
	\textbf{Input: }Moment tensor $M\in \Real^{m^\tau\times m^\tau \times m}$ over a window of length $\tau$\\
	\textbf{Output: }Estimates $\hat{T}$ and $\hat{O}$\\
	
	\textbf{Tensor decomposition using simultaneous diagonalization: }
	\begin{enumerate}
		
		\item Choose $a, b\in\Real^{d}$ uniformly at random. Project $M$ along the 3rd dimension to obtain $X, Y$ with $ X_{i,j} = \sum_{k}^{}M_{i,j,k}a_k$ and $Y_{i,j} = \sum_{k}^{}M_{i,j,k}b_k$. 
		\item Compute the eigendecomposition of $X(Y)^{-1}$ and $Y(X)^{-1}$. Let the columns of $A$ and $B$ to be the eigenvectors of $X(Y)^{-1}$ and $Y(X)^{-1}$ respectively. Pair them corresponding to reciprocal eigenvalues, and scale $A$ and $B$ to be column-stochastic.
		\item Let ${M}_{(3)}\in \Real^{d^{2\tau}\times d}$ be the mode 3 matricization of $M$. Set ${C= {M}_{(3)}((A\odot B)^\dagger)^T}$.
	\end{enumerate}
	
	\textbf{Estimating $T$ and $O$ from tensor factors: }
	\begin{enumerate}
		\item Estimate $O$ by normalizing $C$ to be stochastic, i.e. $\hat{O}_{[:,i]} = C_{:,i}/(e^{T}C_{[:,i]})$ for all $i$.
		\item Marginalize $A$ over the final time step to obtain $A^{(\tau-1)}$.
		\item Estimate ${T}= (O\odot A^{(\tau-1)})^\dagger A$, 
	\end{enumerate}
\end{algorithm*}

\subsection{High-level proof strategy} \label{coupling}
As the transition and observation matrices can be recovered from the factors of the tensors, our goal is to analyze the conditions under which the tensor decomposition step works provably. Note that the factor matrix $A$ is the likelihood of observing each sequence of observations conditioned on starting at a given hidden state. We'll refer to $A$ as the \emph{likelihood matrix} for this reason. $B$ is the equivalent matrix for the time-reversed Markov chain. If we show that $A, B$ are full rank and no two columns of $C$ are the same, then the tensor has a unique decomposition (Kruskal's condition \cite{kruskal1977three}), and the HMM can be learned provided the exact moments using the simultaneous diagonalization algorithm (see Algorithm \ref{alg:learn_tensor}). We show this property for our identifiability results. For our learnability results, we show that the matrices $A$ and $B$ are well-conditioned (have condition numbers polynomial in $n$), which implies learnability from polynomial samples. This is the main technical contribution of the paper, and requires analyzing the condition number of the Khatri-Rao product of dependent random vectors. Before sketching the argument, we first introduce some notation. We can define $A^{(t)}$ as the likelihood matrix over $t$ steps:
\vspace{-4pt}$$A_{L(l_1^{t}),i}^{(t)} = \Prob[y_1^{t}= l_1^{t} \mid h_0 = i].$$
$A^{(t)}$ can be recursively written down as follows:
\vspace{-4pt}
\begin{align}
A^{(0)}=OT, \; A^{(t)}=(O\odot A^{(t-1)})T \label{eq:rec}
\end{align}
where $A\odot B$, denotes the Khatri-Rao product of the matrices $A$ and $B$. If $A$ and $B$ are two matrices of size $m_1 \times r$ and $m_2 \times r$ then the Khatri-Rao product is a $m_1 m_2 \times r$ matrix whose $i$th column is the outer product $A_i \otimes B_i$ flattened into a vector. Note that  $A^{(\tau)}$ is the same as $A$. We now sketch our argument for showing that $A^{(\tau)}$ is well-conditioned under appropriate conditions.
\vspace{-6pt}
\paragraph{Coupling random walks to analyze the Khatri-Rao product.}

As mentioned in the introduction, in this paper we are interested in the setting where the transition matrix is fixed but the observation matrix is drawn at random. If we could draw fresh random matrices $O$ at each time step of the recursion in Eq.~\ref{eq:rec}, then $A$ would be well-conditioned by the smoothed analysis of the Khatri-Rao product due to \citet{bhaskara2014smoothed}. However, our setting is significantly more difficult, as we do not have access to fresh randomness at each time step, so the techniques of \citet{bhaskara2014smoothed} cannot be applied here. As pointed out earlier, the condition number of $A$ in this scenario depends crucially on the transition matrix $T$, as $A$ is not even full rank if $T=I$.

Instead, we analyze $A$ by a coupling argument. To get some intuition for this, note that if $A$ does not have full rank, then there are two disjoint sets of columns of $A$ whose linear combinations are equal, and these combination weights can be used to setup the initial states of two random walks defined by the transition matrix $T$ which have the same output distribution for $\tau$ time steps. More generally, if $A$ is ill-conditioned then there are two random walks with disjoint starting states which have very similar output distributions. We show that if two random walks have very similar output distributions over $\tau$ time steps for a randomly chosen observation matrix $O$, then most of the probability mass in these random walks can be coupled. On the other hand, if $(\sigma_{\min}^{(1)}(T))^\tau$ is sufficiently large, the total variational distance between random walks starting at two different starting states must be at least $(\sigma_{\min}^{(1)}(T))^\tau$ after $\tau$ time steps, and so there cannot be a good coupling, and $A$ is well-conditioned. We provide a sketch of the argument for a simple case in Section \ref{sec:learnability}.
\subsection{Illustrative examples} \label{sec:examples}
\begin{figure}
	\centering
	\begin{subfigure}{0.4 \textwidth}
		\centering
		\includegraphics[height=1 in]{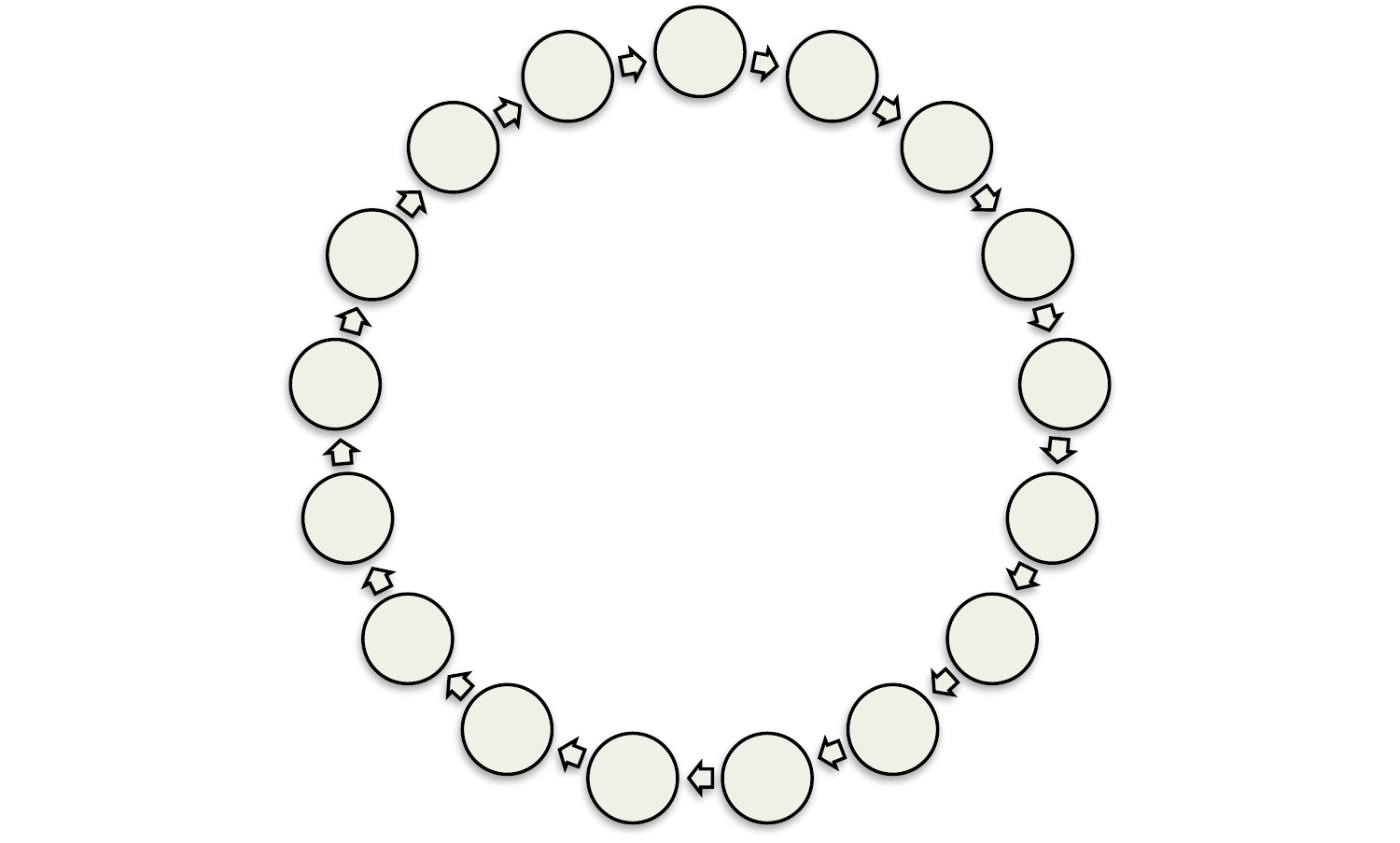}
		\caption{Transition matrix is a cycle, or a \\permutation on the hidden states.}
		\label{fig:example2}
	\end{subfigure}
	\begin{subfigure}{0.4 \textwidth}
		\centering
		\includegraphics[height=1 in]{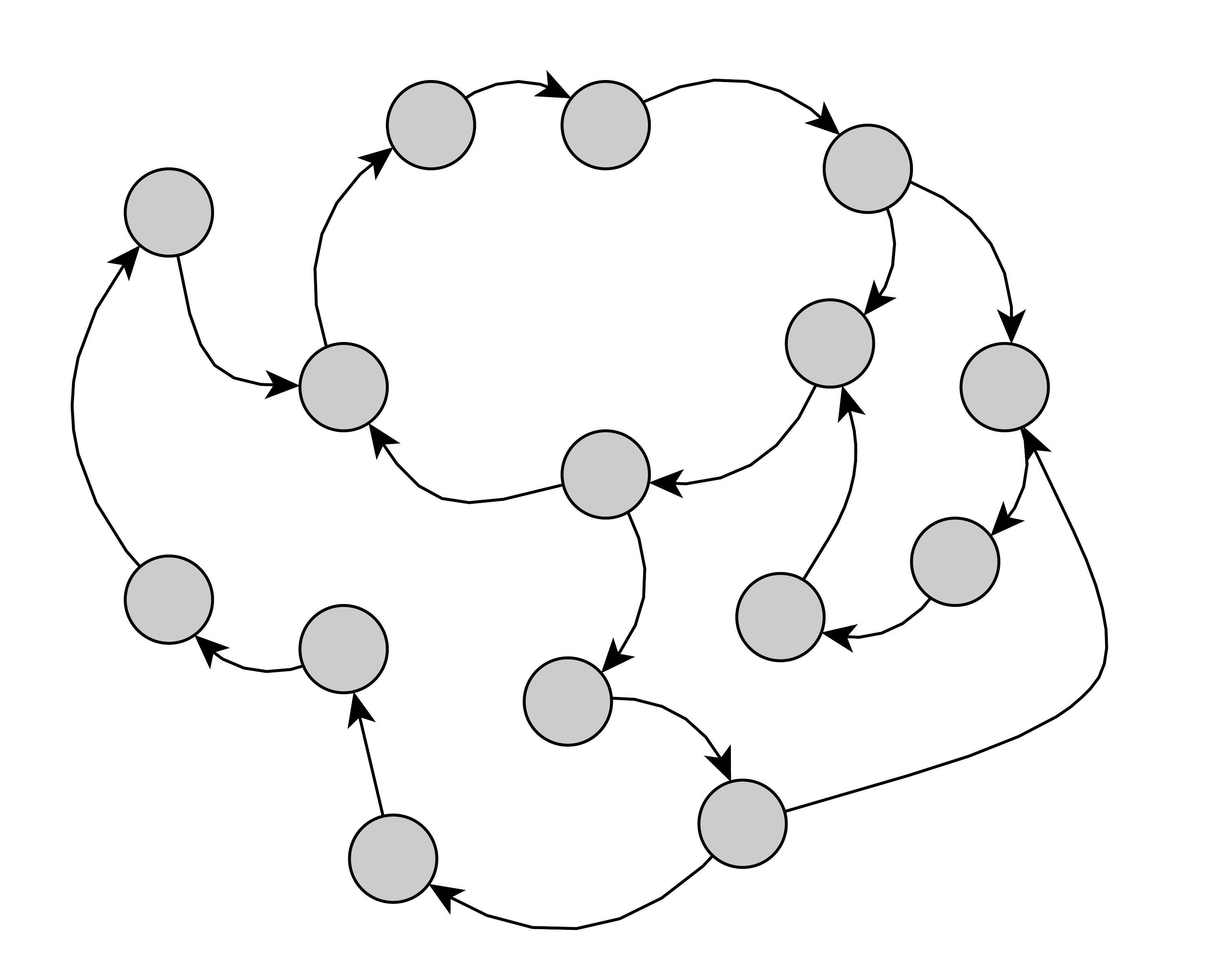}
		\caption{Transition matrix is a random walk on a graph with small degree and no short cycles.}
		\label{fig:example1}
	\end{subfigure}
	\caption{Examples of transition matrices which we can learn, refer to Section \ref{sec:examples} and Section \ref{subsec:examples}.}
	\label{fig:examples}
\end{figure}

We now provide a few simple examples which will illustrate some classes of HMMs we can and cannot learn. We first provide an example of a class of simple HMMs which can be handled by our results, but has non-generic transition matrices and hence does not fit into the framework of \citet{huang14hmm}. Consider an HMM where the transition matrix is a permutation or cyclic shift on the hidden states (see Fig. \ref{fig:example2}).  Our results imply that such HMMs are learnable in polynomial time from polynomial samples if the output distributions of the hidden states are chosen at random. We will try to provide some intuition about why an HMM with the transition matrix as in Fig. \ref{fig:example2} should be efficiently learnable. Let us consider the simple case when the outputs are binary (so $m=2$) and each hidden state deterministically outputs a 0 or a 1, and is labeled by a 0 or a 1 accordingly. If the labels are assigned at random, then with high probability the string of labels of any continuous sequence of $2\log_2 n$ hidden states in the cycle in Fig. \ref{fig:example2} will be unique. This means that the output distribution in a $2\log_2 n$ time window is unique for every initial hidden state, and it can be shown that this ensures that the moment tensor has a unique factorization. By showing that the output distribution in a $2\log_2 n$ time window is very different for different initial hidden states---in addition to being unique---we can show that the factors of the moment tensor are well-conditioned, which allows recovery with efficient sample complexity. As another slightly more complex example of an HMM we can learn, Fig. \ref{fig:example1} depicts an HMM whose transition matrix is a random walk on a graph with small degree and no short cycles. Our learnability result can handle such HMMs having structured transition matrices.

As an example of an HMM which cannot be learned in our framework, consider an HMM with transition matrix $T=I$ and binary observations ($m=2$), see Fig. \ref{fig:nonex1}. In this case, the probability of an output sequence only depends on the total number of zeros or ones in the sequence. Therefore, we only get $t$ independent measurements from windows of length $t$, hence windows of length $\Oh{n}$ instead of $\Oh{\log_2 n}$ are necessary for identifiability (also refer to \citet{blischke1964estimating} for more discussions on this case). More generally, we prove in Proposition \ref{lem:iden_lower_bnd} that for small $m$ a transition matrix composed only of cycles of constant length (see Fig. \ref{fig:nonex2}) requires the window length to be polynomial in $n$ to become identifiable.
\begin{restatable}{prop}{idenbound}\label{lem:iden_lower_bnd}
  Consider an HMM on $n$ hidden states and $m$ observations with the transition matrix being a permutation composed of cycles of length $c$. Then windows of length $O(n^{1/{m^c}})$ are necessary for the model to be identifiable, which is polynomial in $n$ for constant $c$ and $m$.
\end{restatable}
\vspace{-4pt}
The root cause of the difficulty in learning HMMs having short cycles is that they do not visit a large enough portion of the state space in $\Oh{\log_m n}$ steps, and hence moments over a $\Oh{\log_m n}$ time window do not carry sufficient information for learning. Our results cannot handle such classes of transition matrices, also see Section \ref{subsec:assumptions} for more discussion. 
\begin{figure}
	\centering
	\begin{subfigure}{0.4 \textwidth}
		\centering
		\includegraphics[height=1 in]{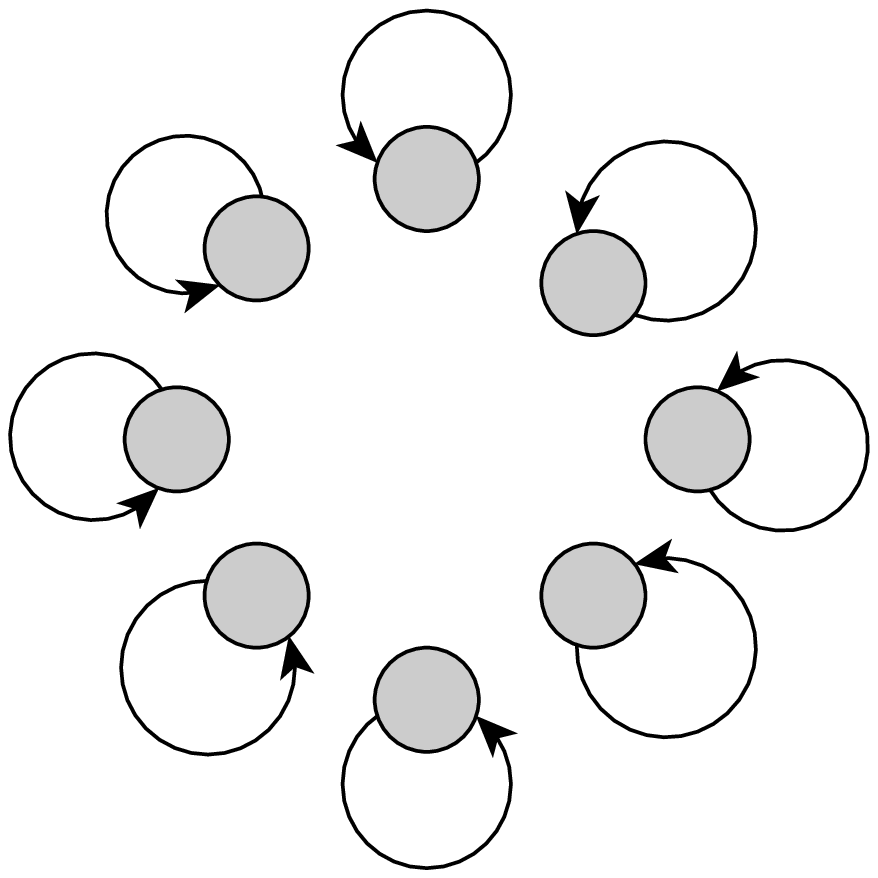}
		\caption{Transition matrix is the identity on\\ 8 hidden states.}
		\label{fig:nonex1}
	\end{subfigure}
	\begin{subfigure}{0.4 \textwidth}
		\centering
		\includegraphics[height=1 in]{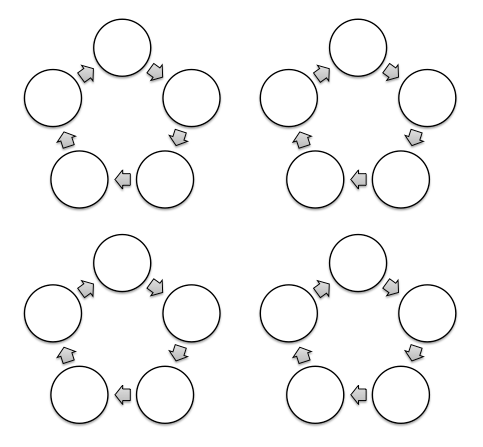}
		\caption{Transition matrix is a union of 4 cycles, each on 5 hidden states.}
		\label{fig:nonex2}
	\end{subfigure}
	\caption{Examples of transition matrices which do not fit in our framework. Proposition \ref{lem:iden_lower_bnd} shows that such HMMs where the transition matrix is composed of a union of cycles of constant length are not even identifiable from short windows of length $\Oh{\log_m n}$}
\end{figure}


\vspace{-8pt}
\section{Learnability results for overcomplete HMMs}\label{sec:learnability}
\vspace{-8pt}




In this section, we state our learnability result, discuss the assumptions and provide examples of HMMs which satisfy these assumptions. Our learnability results hold under the following conditions:

\noindent \textbf{Assumptions:} For fixed constants $c_1, c_2, c_3 > 1 $, the HMM satisfies the following properties for some $c>0$:
\begin{enumerate}
	\item \emph{Transition matrix is well-conditioned:} Both $T$ and the transition matrix $T'$ of the time reversed Markov Chain are well-conditioned in the $\ell_1$-norm: $\sigma_{\min}^{(1)}(T),\sigma_{\min}^{(1)}(T') \ge 1/m^{ c/c_1}$
	\item \emph{Transition matrix does not have short cycles:} For both $T$ and $T'$, every state visits at least $10 \log_m n$ states in $15 \log_m n$ time except with probability $\delta_1 \le 1/n^c$.
	\item \emph{All hidden states have small ``degree'':} There exists $\delta_2$ such that for every hidden state $i$, the transition distributions $T_i$ and $T'_i$ have cumulative mass at most $\delta_2$ on all but $d$ states, with $d \le m^{1/c_2}$ and $\delta_2 \le 1/n^c$. Hence this is a soft ``degree'' requirement.
  \item \emph{Output distributions are random and have small support 
  	:} There exists $\delta_3$ such that for every hidden state $i$ the output distribution $O_i$ has cumulative mass at most $\delta_3$ on all but $k$ outputs, with $k\le m^{1/c_3}$ and $\delta_3 \le 1/n^c$. Also, the output distribution $O_i$ is drawn uniformly on these $k$ outputs.
\end{enumerate}

The constants $c_1, c_2, c_3$ are can be made explicit, for example, $c_1 = 20, c_2 = 16$ and $c_3 = 10$ works. Under these conditions, we show that HMMs can be learned using polynomially many samples:

\begin{restatable}{thm}{learn}\label{thm:learnability}
	If an HMM satisfies the above conditions, then with high probability over the choice of $O$, the parameters of the HMM are learnable to within additive error $\epsilon$ with observations over windows of length $2\tau +1, \tau = 15\log_m n$, with the sample complexity $\poly{n,1/\epsilon}$.
\end{restatable}
\noindent \emph{Proof sketch.} We refer the reader to Section \ref{coupling} for the high level idea. Here, we provide a proof sketch for a much simpler case than that considered in Theorem \ref{thm:learnability} (also see Fig. \ref{fig:coupling}). Recall that our main goal is to show that the likelihood matrix $A$ is well-conditioned. Assume for simplicity that the output distribution of each hidden state is deterministic so the output distribution only has support on one of the $m$ character. The character on which the output distribution of each hidden state is supported is assigned independently and uniformly at random from the output alphabet. Also assume that $\delta_1, \delta_2, \delta_3$ in the conditions for Theorem \ref{thm:learnability} are zero. Our proof steps are roughly as follows--
\begin{enumerate}
	\item Consider two random walks $m_1$ and $m_2$ on $T$ starting at disjoint sets of hidden states at time 0.
	\item We first show that any two sample paths of a random walk on $T$ over $\tau = 15\log_m n$ time steps, both of which visit $10\log_m n$ different states in $\tau$ time steps but never meet in $\tau$ time steps, emit a different sequence of observations with high probability over the randomness in $O$.
	\item Using the fact that the degree of each hidden state is small, we perform a union bound over all possible sample paths to show that with high probability over the choice of $O$, any two sample paths which do not meet in $\tau$ time steps emit a different sequence of observations.
	\item Consider any 2 sample paths $s_1$ and $s_2$ corresponding to the random walks $m_1$ and $m_2$ which emit the same sequence of observations $w$ over $\tau$ time steps. By point 3 above, they must meet at some time $t$. If the probability of emitting $w$ under the random walks $m_1$ and $m_2$ are $p_1$ and $p_2$ respectively and $p_1>p_2$, then we show that $(p_1-p_2)$ of the probability mass in $m_1$ can be coupled with $m_2$ as these sample paths intersect sample paths from $m_2$. This is the core of the argument. Also refer to Fig. \ref{fig:coupling}.
	\item Hence if the probability of emitting a sequence of observations $w$ under the random walks $m_1$ and $m_2$ is very similar for every sequence $w$, then there is a very good coupling of the random walks, which implies that the total variational distance between the distribution of the random walks after $\tau$ time steps must be small. But this is a contradiction as $(\sigma_{\min}^{(1)}(T))^\tau$ is large. The contradiction stems from the fact that the $\ell_1$ distance between $m_1$ and $m_2$ at time 0 is one (as they start at disjoint starting states) and hence the distance at time $\tau$ is at least $ (\sigma_{\min}^{(1)}(T))^\tau $.
\end{enumerate}\hfill $\square$
\begin{figure}[h]
	\centering
	\includegraphics[height=1.5in]{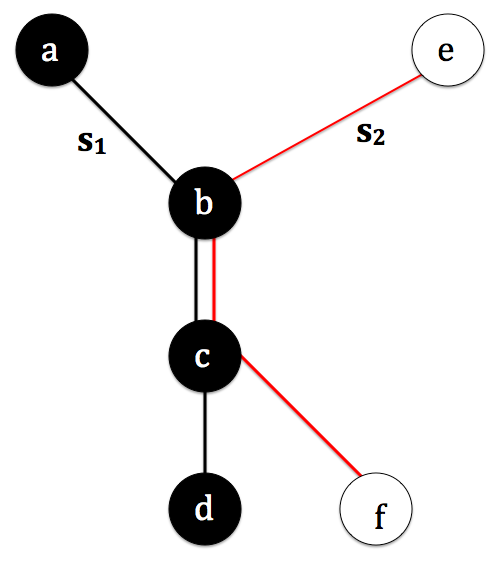}
	\caption{Consider two random walks $m_1$ and $m_2$ for 4 time steps with disjoint starting states and with sample paths $s_1$ and $s_2$ which visits the states \emph{\{a,b,c,d\}} and \emph{\{e,b,c,f\}} at times $\{0,1,2,3\}$ respectively. We show that any two sample paths that have the same output distribution must be at the same hidden state at some time step. For example, here $s_1$ and $s_2$ are simultaneously at states \emph{b} and \emph{c}. This means that the probability mass in the two random walks can be coupled, hence the variational distance between the random walks $m_1$ and $m_2$ must be small at the end. But this cannot be the case as $T$ is well-conditioned. Hence most sample paths of $m_1$ and $m_2$ must have different output distributions, which means that random walks $m_1$ and $m_2$ which start at disjoint states must have different output distributions, which implies that $A$ is well-conditioned.}
	\label{fig:coupling}
\end{figure}

Appendix \ref{sec:learning_proof} also states a corollary of Theorem \ref{thm:learnability} in terms of the minimum singular value $\sigma_{\min}(T)$ of the matrix $T$, instead of $\sigma_{\min}^{(1)}(T)$. We discuss the conditions for Theorem \ref{thm:learnability} next, and subsequently provide examples of HMMs which satisfy these conditions. 


\subsection{Discussion of the assumptions}\label{subsec:assumptions}


\begin{figure}
	\centering
	\begin{subfigure}{0.4 \textwidth}
		\includegraphics[height=1.5in]{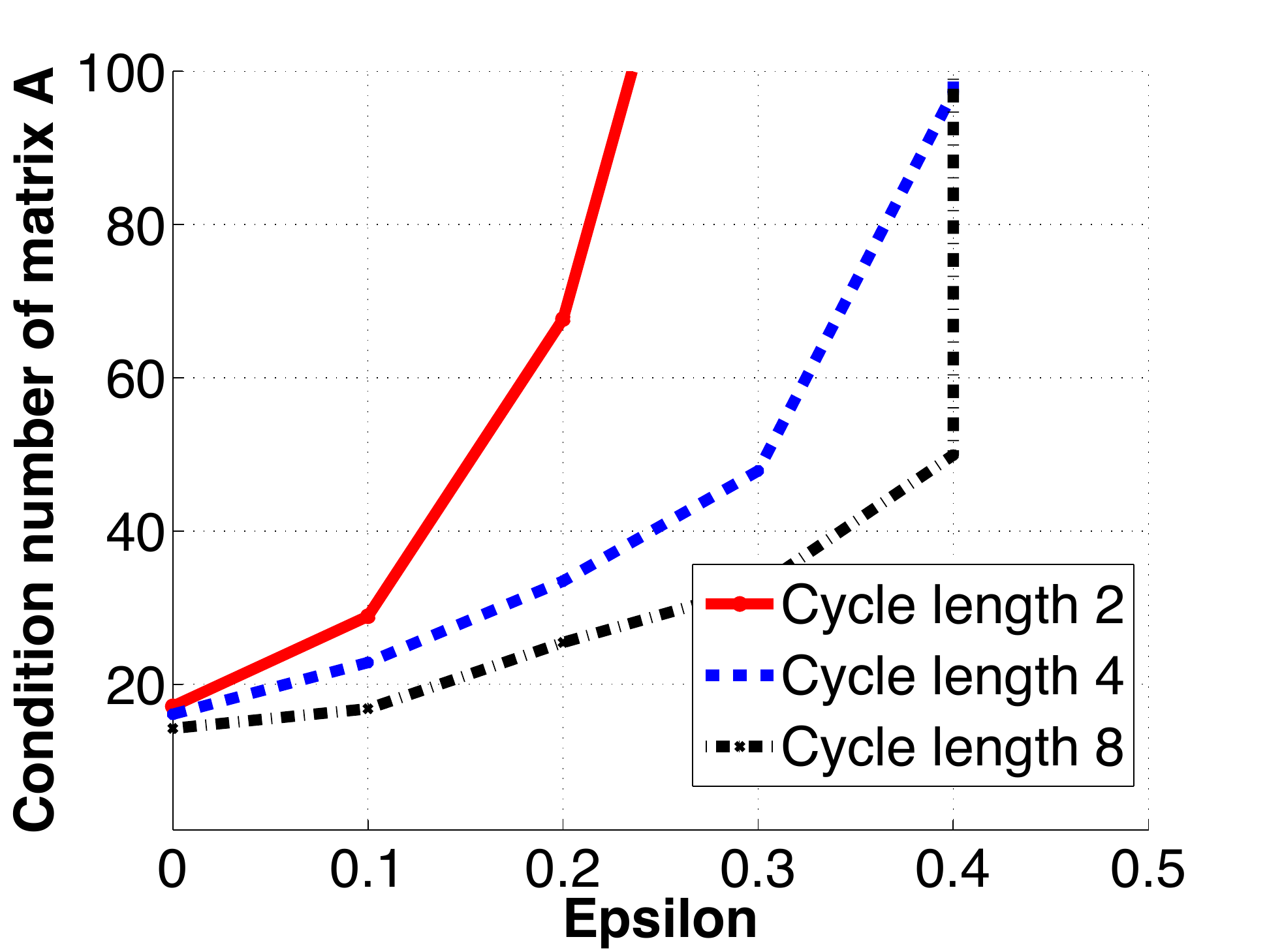}
		\caption{The conditioning becomes worse\\ when cycles are smaller or when more\\ probability mass $\epsilon$ is put on short cycles.}
		\label{fig:cycle_cond}
	\end{subfigure}
	\begin{subfigure}{0.4 \textwidth}
		\includegraphics[height=1.5 in]{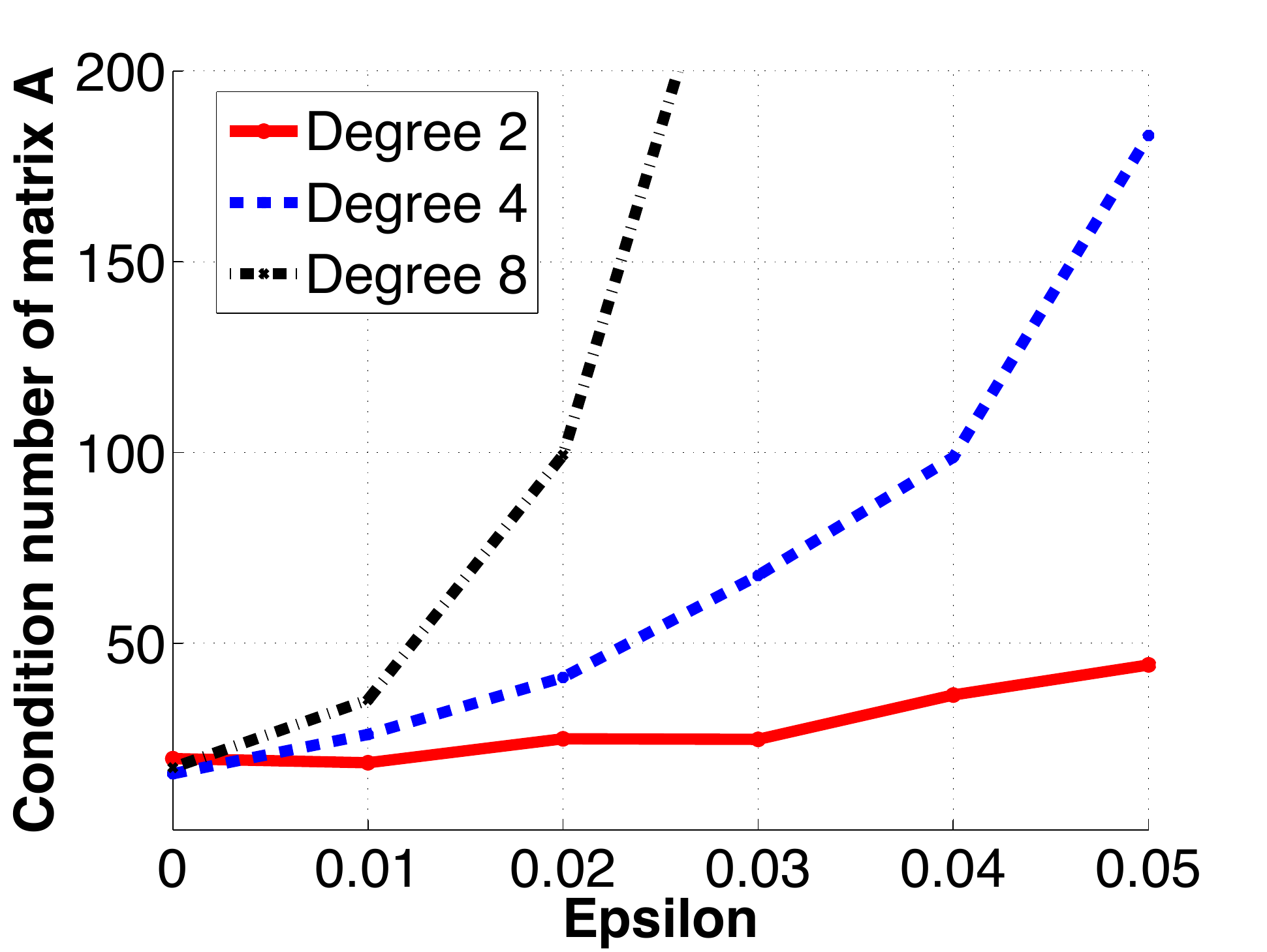}
		\caption{The conditioning becomes worse as the degree increases, and when more probabiltiy mass $\epsilon$ is put on the dense part of $T$.}
		\label{fig:degree_cond}
	\end{subfigure}
	\caption{Experiments to study the effect of sparsity and short cycles on the learnability of HMMs. The condition number of the likelihood matrix $A$ determines the stability or sample complexity of the method of moments approach. The condition numbers are averaged over 10 trials.}
\end{figure}

	
		\,1.\, \emph{Transition matrix is well-conditioned:}
		Note that singular transition matrices might not even be identifiable. Moreover, \citet{mossel2005learning} showed that learning HMMs with singular transition matrices is as hard as learning parity with noise, which is widely conjectured to be computationally hard. Hence, it is necessary to exclude at least some classes of ill-conditioned transition matrices.
	
	\,2.\, \emph{Transition matrix does not have short cycles:} Due to Proposition \ref{lem:iden_lower_bnd}, we know that a HMM might not even be identifiable from short windows if it is composed of a union of short cycles, hence we expect a similar condition for learning the HMM with polynomial samples; though there is a gap between the upper and lower bounds in terms of the probability mass which is allowed on the short cycles. We performed some simulations to understand how the length of cycles in the transition matrix and the probability mass assigned to short cycles affects the condition number of the likelihood matrix $A$; recall that the condition number of $A$ determines the stability of the method of moments approach. We take the number of hidden states $n=128$, and let $P_{128}$ be a cycle on the $n$ hidden states (as in Fig. \ref{fig:example2}). Let $P_c$ be a union of short cycles of length $c$ on the $n$ states (refer to Fig. \ref{fig:nonex2} for an example). We take the transition matrix to be $T=\epsilon P_c+(1-\epsilon) P_{128}$ for different values of $c$ and $\epsilon$. Fig. \ref{fig:cycle_cond} shows that the condition number of $A$ becomes worse and hence learning requires more samples if the cycles are shorter in length, and if more probability mass is assigned to the short cycles, hinting that our conditions are perhaps not be too stringent.
	
	\,3.\, \emph{All hidden states have a small degree:} Condition 3 in Theorem \ref{thm:learnability} can be reinterpreted as saying that the transition probabilities out of any hidden state must have mass at most $1/n^{1+c}$ on any hidden state except a set of $d$ hidden states, for any $c>0$. While this soft constraint is weaker than a hard constraint on the degree, it natural to ask whether any sparsity is necessary to learn HMMs. As above, we carry out simulations to understand how the degree affects the condition number of the likelihood matrix $A$. We consider transition matrices on $n=128$ hidden states which are a combination of a dense part and a cycle. Define $P_{128}$ to be a cycle as before. Define $G_{d}$ as the adjacency matrix of a directed regular graph with degree $d$. We take the transition matrix $T=\epsilon G_d +(1-\epsilon d)P_{128}$. Hence the transition distribution of every hidden state has mass $\epsilon$ on a set of $d$ neighbors, and the residual probability mass is assigned to the permutation $P_{128}$. Fig. \ref{fig:degree_cond} shows that the condition number of $A$ becomes worse as the degree $d$ becomes larger, and as more probability mass $\epsilon$ is assigned to the dense part $G_d$ of the transition matrix $T$, providing some weak  evidence for the necessity of Condition 3. Also, recall that Theorem \ref{thm:lower_bound} shows that HMMs where the transition matrix is a random walk on an undirected regular graph with large degree (degree polynomial in $n$) cannot be learned using polynomially many samples if $m$ is a constant with respect to $n$. However, such graphs have all eigenvalues except the first one to be less than $O(1/\sqrt{n})$, hence it is not clear if the hardness of learning depends on the large degree itself or is only due to $T$ being ill-conditioned. More concretely, we pose the following open question:
	
		\textbf{Open question:} Consider an HMM with a transition matrix $T=(1-\epsilon)P+\epsilon U$, where $P$ is the cyclic permutation on $n$ hidden states (such as in Fig. \ref{fig:example2}) and $U$ is a random walk on a undirected, regular graph with large degree (polynomial in $n$) and $\epsilon>0$ is a constant. Can this HMM be learned using polynomial samples when $m$ is small (constant) with respect to $n$? This example approximately preserves $\sigma_{\min}(T)$ by the addition of the permutation, and hence the difficulty is only due to the transition matrix having large degree.
	
	
	 \,4.\,\emph{Output distributions are random and have small support:} As discussed in the introduction, if we do not assume that the observation matrices are random, then even simple HMMs with a cycle or permutation as the transition matrix might require long windows even to become identifiable, see Fig. \ref{fig:worsrt_O}. Hence some assumptions on the output distribution do seem necessary for learning the model from short time windows, though our assumptions are probably not tight. For instance, the assumption that the output distributions have a small support makes learning easier as it leads to the outputs being more discriminative of the hidden states, but it is not clear that this is a necessary assumption. Ideally, we would like to prove our learnability results under a \emph{smoothed} model for $O$, where an adversary is allowed to see the transition matrix $T$ and pick any worst-case $O$, but random noise is then added to the output distributions, which limits the power of the adversary. We believe our results should hold under such a smoothed setting, but set this aside for future work.

	
	\begin{figure}[h]
		\centering
		\centering
		\includegraphics[height=1in]{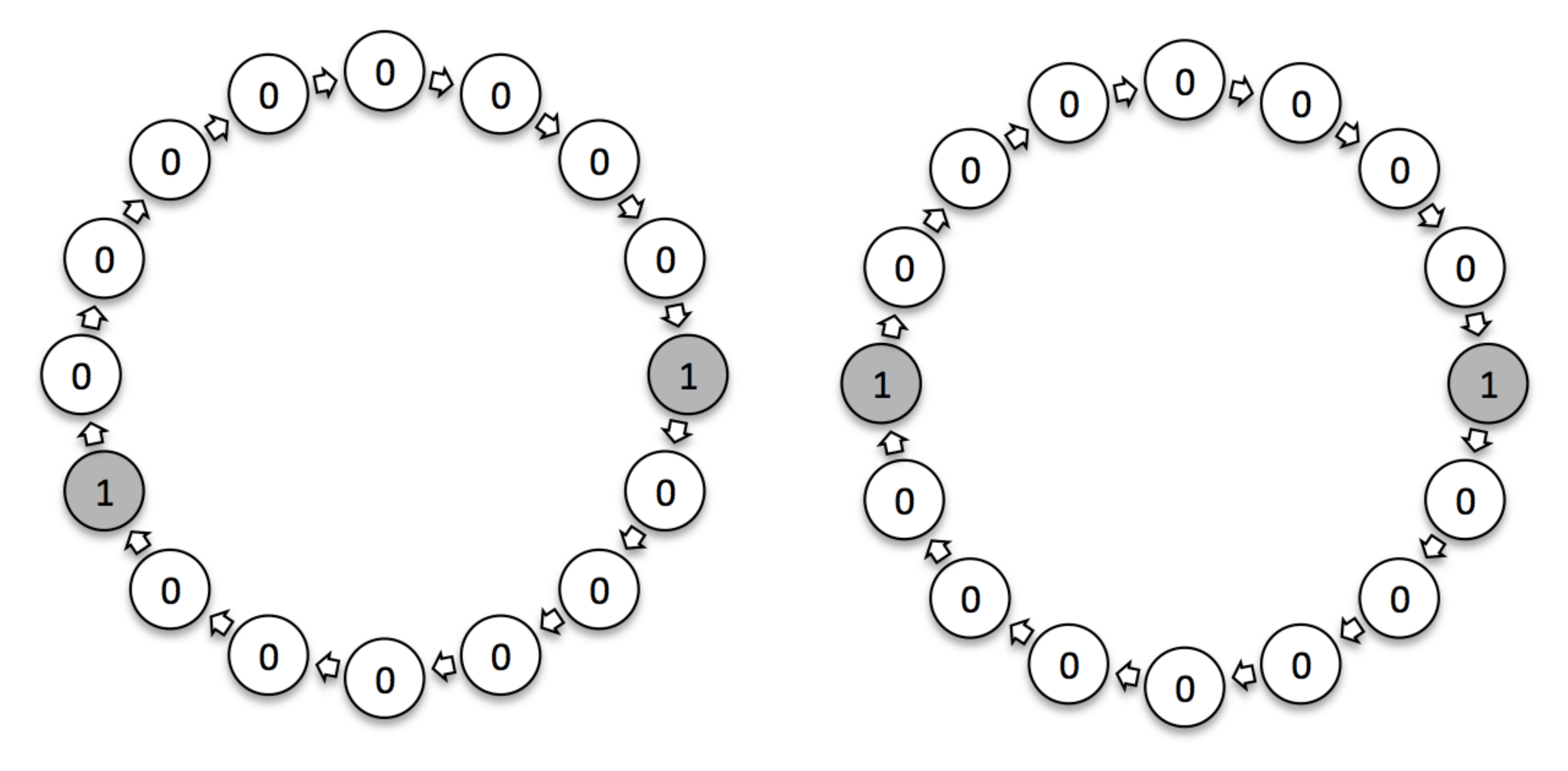}
		\caption{Consider two HMMs with transition matrices being cycles on $n=16$ states with binary outputs, and outputs conditioned on the hidden states are deterministic. The states labeled as 0 always emit a 0 and the states labeled as 1 always emit a 1. The two HMMs are not distinguishable from windows of length less than 8. Hence with worst case $O$ even simple HMMs like the cycle could require long windows to even become identifiable.}
		\label{fig:worsrt_O}
	\end{figure}
	
	
\subsection{Examples of transition matrices which satisfy our assumptions}\label{subsec:examples}

We revisit the examples from Fig. \ref{fig:example2} and Fig. \ref{fig:example1}, showing that they satisfy our assumptions. 

	\,1.\, \emph{Transition matrices where the Markov Chain is a permutation:} If the Markov chain is a permutation with all cycles longer than $10\log_m n $ then the transition matrix obeys all the conditions for Theorem \ref{thm:learnability}. This is because all the singular values of a permutation are 1, the degree is 1 and all hidden states visit $10\log_m n $ different states in $15\log_m n $ time steps.
	
	\,2.\, \emph{Transition matrices which are random walks on graphs with small degree and large girth:} For directed graphs, Condition 2 can be equivalently stated as that the graph representation of the transition matrix has a large girth (girth of a graph is defined as the length of its shortest cycle). 
	
		
	
	\,3.\, \emph{Transition matrices of factorial HMMs:} Factorial HMMs \cite{ghahramani1997factorial} 
	factor the latent state at any time into $D$ dimensions, each of which independently evolves according to a Markov process (see Fig. \ref{fig:factorial}). For $D=2$, this is equivalent to saying that the hidden states are indexed by two labels $(i,j)$ and if $T_1$ and $T_2$ represent the transition matrices for the two dimensions, then $\Prob[(i_1,j_1)\rightarrow(i_2,j_2)]= T_1(i_2,i_1)T_2(j_2,j_1)$. This naturally models settings where there are multiple latent concepts which evolve independently. The following properties are easy to show:
\vspace{-4pt}
		\begin{enumerate}
		\item If either of $T_1$ or $T_2$ visit $N$ different states in $15\log_m n $ time steps with probability $(1-\delta)$, then $T$ visits $N$ different states in $15\log_m n $ time steps with probability $(1-\delta)$.
		\item $\sigma_{\min}(T) = \sigma_{\min}(T_1)\sigma_{\min}(T_2)$
		\item If all hidden states in $T_1$ and $T_2$ have mass at most $\delta$ on all but $d_1$ states and $d_2$ states respectively, then $T$ has mass at most $2\delta$ on all but $d_1d_2$ states.
	\end{enumerate}
	\vspace{-4pt}
	Therefore, factorial HMMs are learnable with random $O$ if the underlying processes obey conditions similar to the assumptions for Theorem \ref{thm:learnability}. If both $T_1$ and $T_2$ are well-conditioned and at least one of them does not have short cycles, and either has small degree, then $T$ is learnable with random $O$.

	\begin{figure}[h]
			\centering
			\includegraphics[height=1.5in]{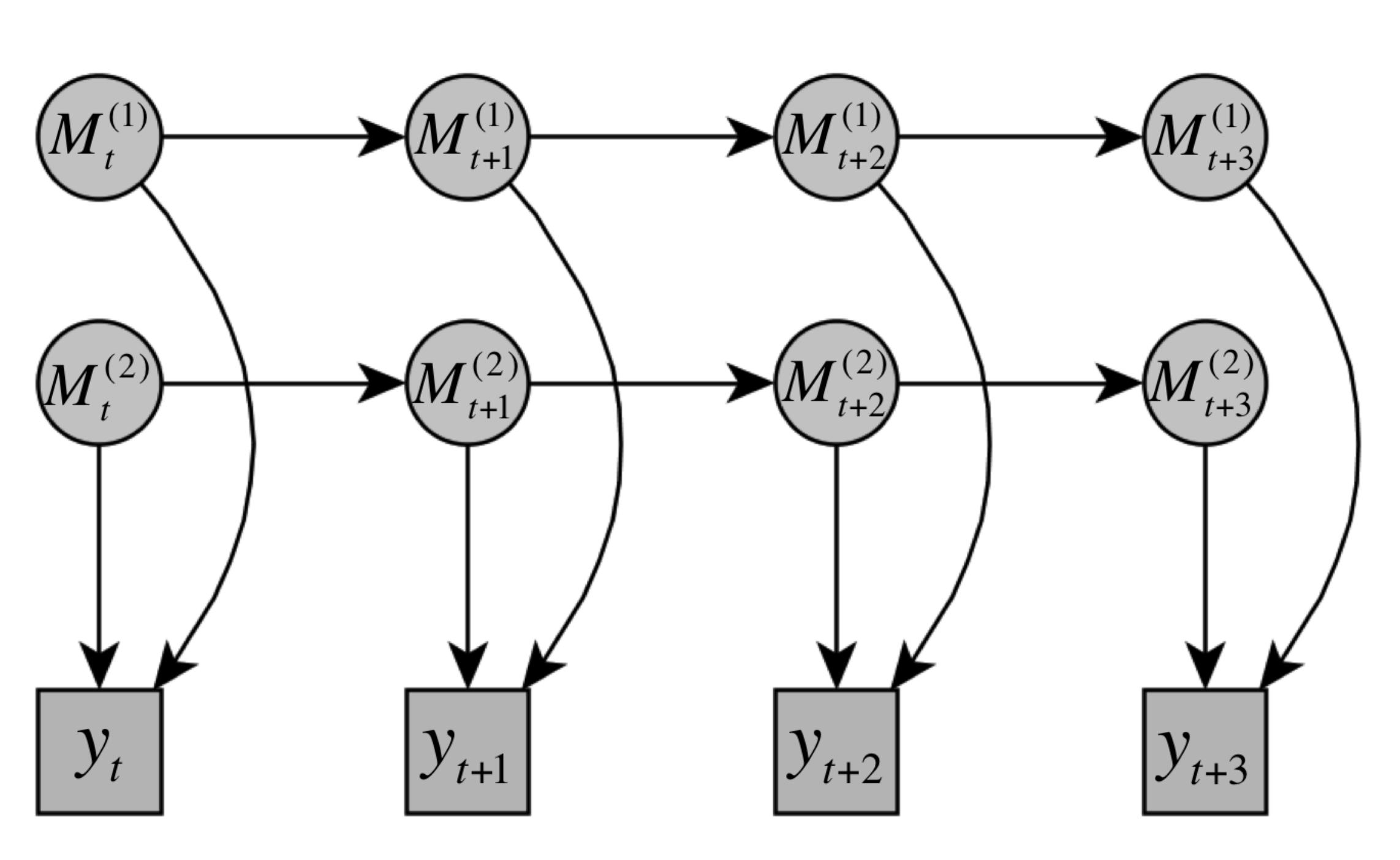}
			\caption{Graphical model for a factorial HMM for $D=2$. The Markov chains $M^{(1)}$ and $M^{(2)}$ evolve independently, and the output at any time step is only dependent on the current states of the two Markov chains at that time step. Our conditions for learning such transition matrices transfer cleanly to conditions on the transition matrices of the underlying Markov chains $M^{(1)}$ and $M^{(2)}$.}
			\label{fig:factorial}
	\end{figure}

%


\vspace{-8pt}
\section{Identifiability of HMMs from short windows}\label{sec:identifiability}
\vspace{-8pt}
As it is not obvious that some of the requirements for Theorem \ref{thm:learnability} are necessary, it is natural to attempt to derive stronger results for just identifiability of HMMs having structured transition matrices. In this section, we state our results for identifiability of HMMs from windows of size $\Oh{\log_m n}$.  \citet{huang14hmm} showed that all HMMs except those belonging to a measure zero set become identifiable from windows of length $2\tau + 1$ with $\tau=8\lceil \log_m n\rceil$. However, the measure zero set itself might possibly contain interesting classes of HMMs (see Fig. \ref{fig:examples}), for example sparse HMMs also belong to a measure zero set. We refine the identifiability results in this section, and show that a natural sparsity condition on the transition matrix guarantees identifiability from short windows. Given any transition matrix $T$, we regard $T$ as being supported by a set of indices $\mathcal{S}$ if the non-zero entries of $T$ all lie in $\mathcal{S}$. We now state our result for identifiability of sparse HMMs.

\begin{restatable}{thm}{identifiability}\label{thm:identifiability}
  Let $\mathcal{S}$ be a set of indices which supports a permutation where all cycles have at least $2\lceil \log_m n \rceil$ hidden states. Then the set $\mathcal{T}$ of all transition matrices with support $\mathcal{S}$ is identifiable from windows of length $4\lceil \log_m n\rceil +1$ for all observation matrices $O$ except for a measure zero set of transition matrices in $\mathcal{T}$ and observation matrices $O$.
\end{restatable}
\noindent \emph{Proof sketch.} Recall from Section \ref{sec:technical_contri} that the main task is to show that the likelihood matrix $A$ is full rank. The proof uses basic algebraic geometry, and the main idea used is analogous to the following fact about polynomials: either a polynomial is a zero polynomial or it has finitely many roots which will lie in a measure zero set. The determinant of the likelihood matrix $A$ (or of sub-matrices of $A$ if $A$ is rectangular) is a polynomial in the entries of $T$ and $O$, hence we only need to show that the polynomial is not a zero polynomial. To show that a polynomial is not a zero polynomial, it is sufficient to find one instance of the variables which makes the polynomial non-zero. Hence we only need to find some particular $T$ and $O$ such that the determinant is not 0. We find such a $T$ and $O$ using the fact that $\mathcal{S}$ supports a permutation which does not have short cycles.\hfill $\square$

\vspace{-4pt}
We hypothesize that excluding a measure zero set of transition matrices in Theorem \ref{thm:identifiability} should not be necessary as long as the transition matrix is full rank, but are unable to show this. Note that our result on identifiability is more flexible in allowing short cycles in transition matrices than Theorem \ref{thm:learnability}, and is closer to the lower bound on identifiability in Proposition \ref{lem:iden_lower_bnd}.

We also strengthen the result of \citet{huang14hmm} for identifiability of generic HMMs. \citet{huang14hmm} conjectured that windows of length $2\lceil \log_m n\rceil +1$ are sufficient for generic HMMs to be identifiable. The constant 2 is the information theoretic bound as an HMM on $n$ hidden states and $m$ outputs has $\Oh{n^2+nm}$ independent parameters, and hence needs observations over a window of size $2\lceil \log_m n\rceil +1$ to be uniquely identifiable. Proposition \ref{thm:special_identifiability} settles this conjecture, proving the optimal window length requirement for generic HMMs to be identifiable. As the number of possible outputs over a window of length $t$ is $m^t$, the size of the moment tensor in Section \ref{sec:technical_contri} is itself exponential in the window length. Therefore even a factor of 2 improvement in the window length requirement leads to a quadratic improvement in the sample and time complexity. 

\begin{restatable}{prop}{tightwindow}\label{thm:special_identifiability}
	The set of all HMMs is identifiable from observations over windows of length $2\lceil \log_m n\rceil +1$ except for a measure zero set of transition matrices $T$ and observation matrices $O$.
\end{restatable}

\vspace{-12pt}
\section{Discussion on long-term dependencies in HMMs}\label{sec:long_term}
\vspace{-8pt}
In this section, we discuss long-term dependencies in HMMs, and show how our results on overcomplete HMMs improve the understanding of how HMMs can capture long-term dependencies, both with respect to the Markov chain and the outputs. Recall the definition of $\sigma_{\min}^{(1)}(T)$:
\vspace{-5pt}
\begin{align}
\sigma_{\min}^{(1)}(T) = \min_{x\in \mathbb{R}^n}\frac{\absnorm{Tx}}{\absnorm{x}}\nonumber
\end{align}
We claim that if $\sigma_{\min}^{(1)}(T)$ is large, then the transition matrix preserves significant information about the distribution of hidden states at time 0 at a future time $t$, for all initial distributions at time 0. Consider any two distributions $p_0$ and $q_0$ at time 0. Let $p_t$ and $q_t$ be the distributions of the hidden states at time $t$ given that the distribution at time 0 is $p_0$ and $q_0$ respectively. Then the $\ell_1$ distance between $p_t$ and $q_t$ is $\absnorm{p_t-q_t}\ge (\sigma_{\min}^{(1)}(T))^t \absnorm{p_0-q_0}$, verifying our claim. It is interesting to compare this notion with the mixing time of the transition matrix. Defining mixing time as the time until the $\ell_1$ distance between any two starting distributions is at most $1/2$, it follows that the mixing time $\tau_{\text{mix}}\ge 1/\log (1/\sigma_{\min}^{(1)}(T))$, therefore if $\sigma_{\min}^{(1)}(T))$ is large then the chain is slowly mixing. However, the converse is not true---$\sigma_{\min}^{(1)}(T)$ might be small even if the chain never mixes, for example if the graph is disconnected but the connected components mix very quickly. Therefore, $\sigma_{\min}^{(1)}(T)$ is possibly a better notion of the long-term dependence of the transition matrix, as it requires that information is preserved about the past state ``in all directions''.

Another reasonable notion of the long-term dependence of the HMM is the long-term dependence in the output process instead of in the hidden Markov chain, which is the utility of past observations when making predictions about the distant future (given outputs $y_{-\infty},\dots,y_1, y_2, \dots, y_t$, at time $t$ how far back do we need to remember about the past to make a good prediction about $y_t$?). This does not depend in a simple way on the $T$ and $O$ matrices, but we do note that if the Markov chain is fast mixing then the output process can certainly not have long-term dependencies. We also note that with respect to long-term dependencies in the output process, the setting $m\ll n$ seems to be much more interesting than when $m$ is comparable to $n$. The reason is that in the small output alphabet setting we only receive a small amount of information about the true hidden state at each step, and hence longer windows are necessary to infer the hidden state and make a good prediction. 
We also refer the reader to \citet{kakade2016prediction} for related discussions on the memory of output processes of HMMs.

\section{Lower bounds for learning dense, random HMMs}\label{sec:lower_bound}

In this section, we state our lower bound that it is not possible to efficiently learn HMMs where the underlying transition matrix is a random walk on a graph with a large degree and $m$ is small with respect to $n$. We actually show a stronger result than this -- we show that the number of bits of information contained in polynomial number of samples from such an HMM is a negligible fraction of the total number of bits of information needed to specify the transition matrix in these cases, showing that approximate learning is also information theoretically impossible.
\begin{restatable}{thm2}{lowerbound}\label{thm:lower_bound}
	Consider the class of HMMs with $n$ hidden states and $m$ outputs and $m=\polylog{n}$ with the transition matrix chosen to be a $d$-regular graph, with $d=n^\epsilon$ for some $\epsilon>0$. Then at least $\Omega(nd)$ bits of information are needed to specify the choice of the transition matrix. However, if the observation matrix $O$ is randomly chosen such that the columns of $O$ are chosen independently and $\E[O_{ij}]=1/m$ for all $i,j$, then the number of bits of information contained in polynomially many samples over a window of length $N=\poly{n}$ is at most $\tilde{\mathcal{O}}(n)$, with high probability over the choice of $O$, where the $\tilde{\mathcal{O}}$ notation hides polylogarithmic factors in $n$.
\end{restatable}
\noindent \emph{Proof sketch.}
The proof consists of two steps--in the first step we show that the information contained in polynomially many observations over windows of length $\tau=\lfloor \log_m n \rfloor$ is not sufficient to learn the HMM. The proof of this part relies on a counting argument and a lower bound on the number of random regular graphs with a given degree. We then show that the information contained in polynomial samples over longer windows is not much larger than the information contained in polynomial samples over a window length of $\tau$. This is the main technical part, and we need to show that the hidden state at time 0 does not have much influence on the hidden state at time $t$, conditioned on the outputs from time 0 to $t$. The conditioning makes this tricky, as the probabilities of the hidden states no longer evolve under the transition matrix of the Markov chain. We get around this by showing that the probability of the hidden states after conditioning on the observations evolves under a time-inhomogeneous Markov chain, and the transition matrices at every time step are related to the outputs from time 1 to $t$ and the original transition matrix. We analyze the spectrum of the time-inhomogeneous transition matrices to show that the influence of the hidden state at time 0 decays at every step and is small at time $t$. \hfill $\square$

We would like to point out that our techniques to prove the information theoretic lower bound appear to be generally useful for analyzing the influence of the hidden state at time 0 on the hidden state at time $t$, conditioned on the outputs from time 0 to $t$, This is a measure of how much value there is to observations before time 0 for predicting the observation at time $t+1$, conditioned on the intermediate observations from time 0 to $t$. This is a natural notion of the memory of the output process.
\vspace{-10pt}
\section{Conclusion and Future Work}\label{sec:conclusion}
\vspace{-5pt}

The setting where the output alphabet $m$ is much smaller than the number of hidden states $n$ is well-motivated in practice and seems to have several interesting theoretical questions about new lower bounds and algorithms. Though some of our results are obtained in more restrictive conditions than seems necessary, we hope the ideas and techniques pave the way for much sharper results in this setting. Some open problems which we think might be particularly useful for improving our understanding is relaxing the condition on the observation matrix being random to some structural constraint on the observation matrix (such as on its Kruskal rank), and more thoroughly investigating the requirement for the transition matrix being sparse and not having short cycles.

\section*{Acknowledgements}

Sham Kakade acknowledges funding from the Washington Research Foundation for Innovation in Data-intensive Discovery, and the NSF Award CCF-1637360. Gregory Valiant and Sham Kakade acknowledge funding form NSF Award CCF-1703574. Gregory was also supported by NSF CAREER Award CCF-1351108 and a Sloan Research Fellowship. 
\bibliography{references.bib,refdb/all.bib}

\begin{thebibliography}{23}
\providecommand{\natexlab}[1]{#1}
\providecommand{\url}[1]{\texttt{#1}}
\expandafter\ifx\csname urlstyle\endcsname\relax
  \providecommand{\doi}[1]{doi: #1}\else
  \providecommand{\doi}{doi: \begingroup \urlstyle{rm}\Url}\fi

\bibitem[Allman et~al.(2009)Allman, Matias, and
  Rhodes]{allman09identifiability}
E.~S. Allman, C.~Matias, and J.~A. Rhodes.
\newblock Identifiability of parameters in latent structure models with many
  observed variables.
\newblock \emph{Annals of Statistics}, 37:\penalty0 3099--3132, 2009.

\bibitem[Anandkumar et~al.(2012)Anandkumar, Hsu, and
  Kakade]{anandkumar2012method}
A.~Anandkumar, D.~J. Hsu, and S.~M. Kakade.
\newblock A method of moments for mixture models and hidden markov models.
\newblock In \emph{COLT}, volume~1, page~4, 2012.

\bibitem[Anandkumar et~al.(2013)Anandkumar, Ge, Hsu, Kakade, and
  Telgarsky]{anandkumar13tensor}
A.~Anandkumar, R.~Ge, D.~Hsu, S.~M. Kakade, and M.~Telgarsky.
\newblock Tensor decompositions for learning latent variable models.
\newblock \emph{arXiv}, 2013.

\bibitem[Anandkumar et~al.(2015)Anandkumar, Ge, and
  Janzamin]{anandkumar2015learning}
A.~Anandkumar, R.~Ge, and M.~Janzamin.
\newblock Learning overcomplete latent variable models through tensor methods.
\newblock In \emph{COLT}, pages 36--112, 2015.

\bibitem[Bhaskara et~al.(2013)Bhaskara, Charikar, and
  Vijayaraghavan]{bhaskara13tensor}
A.~Bhaskara, M.~Charikar, and A.~Vijayaraghavan.
\newblock Uniqueness of tensor decompositions with applications to polynomial
  identifiability.
\newblock \emph{CoRR}, abs/1304.8087, 2013.

\bibitem[Bhaskara et~al.(2014)Bhaskara, Charikar, Moitra, and
  Vijayaraghavan]{bhaskara2014smoothed}
A.~Bhaskara, M.~Charikar, A.~Moitra, and A.~Vijayaraghavan.
\newblock Smoothed analysis of tensor decompositions.
\newblock In \emph{Proceedings of the 46th Annual ACM Symposium on Theory of
  Computing}, pages 594--603. ACM, 2014.

\bibitem[Blackwell and Koopmans(1957)]{blackwell57identifiable}
D.~Blackwell and L.~Koopmans.
\newblock On the identifiability problem for functions of finite {M}arkov
  chains.
\newblock \emph{Annals of Mathematical Statistics}, 28:\penalty0 1011--1015,
  1957.

\bibitem[Blischke(1964)]{blischke1964estimating}
W.~Blischke.
\newblock Estimating the parameters of mixtures of binomial distributions.
\newblock \emph{Journal of the American Statistical Association}, 59\penalty0
  (306):\penalty0 510--528, 1964.

\bibitem[Chang(1996)]{chang1996full}
J.~T. Chang.
\newblock Full reconstruction of markov models on evolutionary trees:
  identifiability and consistency.
\newblock \emph{Mathematical biosciences}, 137\penalty0 (1):\penalty0 51--73,
  1996.

\bibitem[Flaxman et~al.(2004)Flaxman, Harrow, and Sorkin]{flaxman2004strings}
A.~Flaxman, A.~W. Harrow, and G.~B. Sorkin.
\newblock Strings with maximally many distinct subsequences and substrings.
\newblock \emph{Electron. J. Combin}, 11\penalty0 (1):\penalty0 R8, 2004.

\bibitem[Friedman(2003)]{friedman2003proof}
J.~Friedman.
\newblock A proof of {Alon's} second eigenvalue conjecture.
\newblock In \emph{Proceedings of the thirty-fifth Annual ACM Symposium on
  Theory of Computing}, pages 720--724. ACM, 2003.

\bibitem[Ghahramani and Jordan(1997)]{ghahramani1997factorial}
Z.~Ghahramani and M.~Jordan.
\newblock Factorial hidden markov models.
\newblock \emph{Machine Learning}, 1:\penalty0 31, 1997.

\bibitem[Hsu et~al.(2012)Hsu, Kakade, and Zhang]{hsu2012spectral}
D.~Hsu, S.~M. Kakade, and T.~Zhang.
\newblock A spectral algorithm for learning hidden markov models.
\newblock \emph{Journal of Computer and System Sciences}, 78\penalty0
  (5):\penalty0 1460--1480, 2012.

\bibitem[Huang et~al.(2016)Huang, Ge, Kakade, and Dahleh]{huang14hmm}
Q.~Huang, R.~Ge, S.~Kakade, and M.~Dahleh.
\newblock Minimal realization problems for hidden markov models.
\newblock \emph{IEEE Transactions on Signal Processing}, 64\penalty0
  (7):\penalty0 1896--1904, 2016.

\bibitem[Ito et~al.(1992)Ito, Amari, and Kobayashi]{ito1992identifiability}
H.~Ito, S.-I. Amari, and K.~Kobayashi.
\newblock Identifiability of hidden markov information sources and their
  minimum degrees of freedom.
\newblock \emph{IEEE transactions on information theory}, 38\penalty0
  (2):\penalty0 324--333, 1992.

\bibitem[Kakade et~al.(2016)Kakade, Liang, Sharan, and
  Valiant]{kakade2016prediction}
S.~Kakade, P.~Liang, V.~Sharan, and G.~Valiant.
\newblock Prediction with a short memory.
\newblock \emph{arXiv preprint arXiv:1612.02526}, 2016.

\bibitem[Krivelevich et~al.(2001)Krivelevich, Sudakov, Vu, and
  Wormald]{krivelevich2001random}
M.~Krivelevich, B.~Sudakov, V.~H. Vu, and N.~C. Wormald.
\newblock Random regular graphs of high degree.
\newblock \emph{Random Structures \& Algorithms}, 18\penalty0 (4):\penalty0
  346--363, 2001.

\bibitem[Kruskal(1977)]{kruskal1977three}
J.~B. Kruskal.
\newblock Three-way arrays: rank and uniqueness of trilinear decompositions,
  with application to arithmetic complexity and statistics.
\newblock \emph{Linear Algebra and its Applications}, 18\penalty0 (2), 1977.

\bibitem[Leurgans et~al.(1993)Leurgans, Ross, and
  Abel]{leurgans1993decomposition}
S.~Leurgans, R.~Ross, and R.~Abel.
\newblock A decomposition for three-way arrays.
\newblock \emph{SIAM Journal on Matrix Analysis and Applications}, 14\penalty0
  (4):\penalty0 1064--1083, 1993.

\bibitem[Mossel and Roch(2005)]{mossel2005learning}
E.~Mossel and S.~Roch.
\newblock Learning nonsingular phylogenies and hidden markov models.
\newblock In \emph{Proceedings of the thirty-seventh Annual ACM Symposium on
  Theory of Computing}, pages 366--375. ACM, 2005.

\bibitem[Redner and Walker(1984)]{redner1984mixture}
R.~A. Redner and H.~F. Walker.
\newblock Mixture densities, maximum likelihood and the em algorithm.
\newblock \emph{SIAM review}, 26\penalty0 (2):\penalty0 195--239, 1984.

\bibitem[Shamir and Upfal(1984)]{shamir1984large}
E.~Shamir and E.~Upfal.
\newblock Large regular factors in random graphs.
\newblock \emph{North-Holland Mathematics Studies}, 87:\penalty0 271--282,
  1984.

\bibitem[Weiss and Nadler(2015)]{weiss2015learning}
R.~Weiss and B.~Nadler.
\newblock Learning parametric-output hmms with two aliased states.
\newblock In \emph{ICML}, pages 635--644, 2015.

\end{thebibliography}
\bibliographystyle{abbrvnat}
\appendix

\section{Proof of lower bound for dense HMMs}\label{sec:lower_bound_proof}


\begin{thm2}
	Consider the class of HMMs with $n$ hidden states and $m$ outputs and $m=\polylog{n}$ with the transition matrix chosen to be a $d$-regular graph, with $d=n^\epsilon$ for some $\epsilon>0$. Then at least $\Omega(nd)$ bits of information are needed to specify the choice of the transition matrix. However, if the observation matrix $O$ is randomly chosen such that the columns of $O$ are chosen independently and $\E[O_{ij}]=1/m$ for all $i,j$, then the number of bits of information contained in polynomially many samples over a window of length $N=\poly{n}$ is at most $\tilde{\mathcal{O}}(n)$, with high probability over the choice of $O$, where the $\tilde{\mathcal{O}}$ notation hides polylogarithmic factors in $n$.
\end{thm2}
\begin{proof}
	By \citet{shamir1984large} (also see \citet{krivelevich2001random}), the number of $d$-regular graphs on $n$ vertices is at least ${\dbinom{\binom{n}{2}}{ nd/2}}\exp(-nd^{0.5+\delta})$ for any fixed $\delta>0$. This can be bounded from below as follows---
	\begin{align}
	{\dbinom{\binom{n}{2}}{ nd/2}}\exp(-nd^{0.5+\delta})&=\Big(\frac{n-1}{d}\Big)^{nd/2}\exp(-nd^{0.5+\delta})\nonumber\\
	&\ge 2^{nd/2-nd^{0.5+\delta}}\nonumber
	\end{align}
	Hence the number of bits needed to specify a randomly chosen $d$-regular graph on $n$ vertices is at least $\Omega(nd)$. 
	
	Note that if we only get observations over a window of length $\tau = \lfloor \log_m n \rfloor$, and we obtain $\poly{n}$ samples, then the total information in those samples is at most $\tilde{O}(n)$, where the $\tilde{O}$ hides polylogarithmic factors in $n$. This is because there are at most $m^{\tau} \le n$ possible outputs, each of which can take $\poly{n}$ different values as there are $\poly{n}$ samples. We will now show that getting polynomially many samples over windows of length $N=\poly{n}$ is equivalent to getting polynomially many samples over windows of length $\tau$.
	
	For notational convenience, we will refer to $P[o_t=i]$, the probability of the output at time $t$ being $i$, as $P[o_t]$ whenever the assignment to the random variable is clear from the context. The probability of any sequence of outputs $\{o_1,o_2,\dotsb, o_{N}\}$ can be written down as follows using chain rule,
	\begin{align}
	P[o_1,o_2,\dotsb, o_N] &= P[o_1]P[o_2|o_1]P[o_3|o_1,o_2]\dotsb P[o_N|o_{1},\dotsb,o_{N-1}]\nonumber\\
	&=\Pi_{t=1}^N P[o_t|o_1, \dotsb, o_{t-1}]\nonumber
	\end{align}
	In order to prove that the probabilities of sequences of length $N$ can be well-approximated using sequences of length $\tau$, for $t>\tau$, we will approximate the probability $P[o_t|o_{1},\dotsb,o_{t-1}]$ by $P[o_t|o_{t-\tau+1},\dotsb,o_{t-1}]$. Hence our estimate of the probability of sequence $\{o_1,o_2,\dotsb, o_N\}$ is
	\begin{align}
	\hat{P}[\{o_1,o_2,\dotsb, o_N\}] &= \Pi_{t=1}^\tau {P}[o_t|o_{(t-\tau+1)\vee 1}, \dotsb, o_{t-1}]\nonumber
	\end{align}
	where $a\vee b$ denotes $\max(a,b)$. If 
	\begin{align}
	\absnorm{P[o_t|o_1, \dotsb, o_{t-1}]-{P}[o_t| o_{(t-\tau+1)\vee 1}, \dotsb,o_{t-1}]} \le \epsilon \label{eq:est_short}
	\end{align}
	for all $t\le N$ and assignments to $\{o_1, \dotsb, o_{t}\}$, then
	\begin{align}
	\Big|P[o_1,o_2,\dotsb, o_N]-\hat{P}[o_1,o_2,\dotsb, o_N] \Big| \le O(\epsilon N )\nonumber
	\end{align}
	for all assignments to $\{o_1, \dotsb, o_{N}\}$. Hence the probabilities of windows of length $N$ can be estimated from windows of length $\tau$ up to an additive error of $O(\epsilon N )$. Therefore, if we can show that $\epsilon \le o(1/\poly{n})$, then given the true probabilities of windows of length $\tau$, it is possible to estimate the true probabilities over windows of length $N=\poly{n}$ up to an inverse super-polynomial factor of $n$. Given empirical probabilities of windows of length $\tau$ up to an accuracy of $\delta$, it is possible to estimate the true probabilities over windows of length $N$ up to an error of $(\epsilon + \delta)N$. As $\epsilon N$ is inverse super-polynomial in $N$, getting $S$ samples over windows of length $N$ is equivalent to getting $\poly{N,S}$ samples over windows of length $\tau$. Therefore, the information contained in polynomially many samples over windows of length $N$ is entirely contained in the polynomially many samples over windows of length $\tau$ (the polynomials would be different, but this does not concern us as we have shown that the information in polynomially many samples over windows of length $\tau$ is always $\tilde{O}(n)$). Hence the information contained in polynomially many samples over windows of length $N$ can be at most $\tilde{O}(n)$. 
	
	We will now prove Eq. \ref{eq:est_short}. First, note that Eq. \ref{eq:est_short} can be written in terms of the probabilities of the hidden states as follows,
	\begin{align}
	P[o_t=j|o_1, \dotsb, o_{t}] = \sum_{i=1}^{n}P[o_t=j|h_t=i]P[h_t=i|o_1, \dotsb, o_{t}]\nonumber
	\end{align}
	Therefore, it is sufficient to show the following, which says that the distribution of the hidden states conditioned on the two observation windows is similar--
	\begin{align}
	\absnorm{P[h_{t}|o_1, \dotsb, o_{t-1}]-{P}[h_t| o_{(t-\tau+1)\vee 1}, \dotsb,o_{t-1}]} \le \epsilon \label{eq:rtp1}
	\end{align}
	for all $t\le N$ and assignments to $\{o_1, \dotsb, o_{t-1}\}$. Note that we do not need to worry about the case when $t\ge \tau$ as the observation windows under consideration are the same for both terms. We will shift our windows and fix $t=\tau$ to make notation easy, as the process is stationary this can be done without loss of generality. Hence we will rewrite Eq. \ref{eq:rtp1} as follows, ignoring the cases when the terms are the same because $(t-\tau+1)\vee 1=1$. 
	\begin{align}
	\Big|P[h_{\tau}|o_1, \dotsb, o_{\tau-1}]-{P}[h_{\tau}|o_{z}, \dotsb, o_{\tau-1}]\Big| \le \epsilon \nonumber
	\end{align}
	for all $z \in[\tau-N+1,0]$. 
	
	Define the modified transition matrix $T^{(t)}$ as ${T^{(t)}_{i,j}=P[h_{t+1}=j|h_t=i,o_{t+1},\dotsb,o_{\tau-1}]}$. For any $s \in [0,\tau]$, we claim that,
	\begin{align}
	P[h_{s}|o_z, \dotsb, o_{\tau- 1}]=\Big(\Pi_{t=1}^{s} T^{(t)}\Big)P[h_0|o_z, \dotsb, o_{\tau- 1}]\nonumber
	\end{align}
	for all $z \in[\tau-N+1,0]$. Therefore $T^{(t)}$ serves the role of the transition matrix at time $t$ in our setup. The proof of this follows from a simple induction argument on $s$. The base case $s=0$ is clearly correct. Let the statement be true up to some time $p$. Then we can write,
	\begin{align}
	P[h_{p+1}|o_z, \dotsb, o_{\tau-1}]&=\sum_{i}^{}P[h_p=i,h_{p+1}|o_z, \dotsb, o_{\tau-1}]\nonumber\\
	&=\sum_{i}^{}P[h_p=i|o_z, \dotsb, o_{\tau-1}]P[h_{p+1}|h_p=i,o_z, \dotsb, o_{\tau}]\nonumber\\
	&=\sum_{i}^{}P[h_{p+1}|h_p=i,o_{p+1}, \dotsb,o_{\tau-1}]\Big(\Pi_{t=1}^{p} T^{(t)}\Big)P[h_0|o_z, \dotsb, o_{\tau- 1}]\nonumber \\	
	&=\Big(\Pi_{t=1}^{p+1} T^{(t)}\Big)P[h_0|o_z, \dotsb, o_{\tau- 1}]\nonumber
	\end{align}
	where we could simplify $P[h_{p+1}|h_p=i,o_z, \dotsb, o_{\tau-1}]=P[h_{p+1}|h_p=i,o_{p+1}, \dotsb, o_{\tau-1}]$ as conditioned on the hidden state at time $p$, the observations at and before time $p$ do not affect the distribution of future hidden states. Therefore, our task now reduces to analyzing the spectrum of the time-inhomogeneous transition matrices $T^{(t)}$. The following Lemma does this.
	\begin{restatable}{lem}{boundspectrum}\label{lem:bound_spectrum}
		For any $x$ with $\norm{x}=1$ and $1^Tx=0$, $\norm{T^{(t)}x}\le \alpha+\lambda$ where $\alpha \le \sqrt{\frac{100m^3\log^3 n}{2d}}$  and $\lambda <3/\sqrt{d}$. Therefore if $d=n^{\epsilon}$ for some $\epsilon>0$ and $m=\polylog{n}$, then $\norm{T^{(t)}x}\le n^{-\epsilon_1}$ for some $\epsilon_1>0$. 
	\end{restatable}
	Given Lemma \ref{lem:bound_spectrum}, we will show $\epsilon=o(1/\poly{n})$. Let $p_1 = P[h_0|o_1, \dotsb, o_{\tau- 1}]$ and $p_2 = P[h_0|o_z, \dotsb, o_{\tau- 1}]$ for any $z<1$. We can write,
	\begin{align}
	 P[h_0|o_1, \dotsb, o_{\tau- 1}] - P[h_0|o_z, \dotsb, o_{\tau- 1}]  = \Big(\Pi_{t=1}^{\tau} T^{(t)}\Big)(p_1-p_2)\nonumber
	\end{align}
	Let $p_1-p_2 = x$. Note that $\norm{x}\le 1$ and $1^Tx =0$. Furthermore, as $T^{(t)}$ is stochastic for every $t$, therefore $1^T \Big(\Pi_{t=1}^{s} T^{(t)}\Big)x$ is also $0$ for every $s$. Hence we can use Lemma \ref{lem:bound_spectrum} to say,
	\begin{align}
	\norm{P[h_0|o_1, \dotsb, o_{\tau- 1}] - P[h_0|o_t, \dotsb, o_{\tau- 1}]}&\le (2(\alpha+\lambda))^{\tau}\nonumber\\
	\implies \absnorm{P[h_0|o_1, \dotsb, o_{\tau- 1}] - P[h_0|o_t, \dotsb, o_{\tau- 1}]}&\le \sqrt{n}(\alpha+\lambda)^{\tau}\nonumber\\
	&\le n^{- \log^{\delta} n}\nonumber
	\end{align}
	for a fixed $\delta>0$. This is superpolynomial in $n$, proving Theorem \ref{thm:lower_bound}. We will now prove Lemma \ref{lem:bound_spectrum}. 
	\boundspectrum*
	\begin{proof}
		We show that $T^{(t)}$ has a simple decomposition, $T^{(t)}=O^{(t)}TE^{(t)}$, where $O^{(t)}$ is a diagonal matrix with $O^{(t)}_{i,i}=P(o_t,\dotsb,o_{\tau}|h_t=i)$ and $E^{(t)}$ is another diagonal matrix with $E^{(t)}_{i,i}= P[o_{t+1},\dotsb,o_{\tau}|h_{t}=i]$. This is because,
		\begin{align}
		P[h_{t+1}=j|h_t=i,o_{t+1},\dotsb,o_{\tau}] &= \frac{P[h_{t+1}=j|h_t = i]P[o_{t+1},\dotsb,o_{\tau}|h_{t+1}=j]}{\sum_{j}^{}P[h_{t+1}=j|h_t = i]P[o_{t+1},\dotsb,o_{\tau}|h_{t+1}=j]}\nonumber\\
		&= \frac{P[h_{t+1}=j|h_t = i]P[o_{t+1},\dotsb,o_{\tau}|h_{t+1}=j]}{P[o_{t+1},\dotsb,o_{\tau}|h_{t}=i]}\nonumber
		\end{align}
		As $T$ is the normalized adjacency of a $d-$regular graph, the eigenvector corresponding to the eigenvalue 1 is the all ones vector, and all subsequent eigenvectors are orthogonal to the all ones vector. Also, the second eigenvalue and all subsequent eigenvalues of $T$ are at most $3/\sqrt{d}$, due to \citet{friedman2003proof}. 
		
		To analyze $O^{(t)}$ and $E^{(t)}$, we need to first derive some properties of the randomly chosen observation matrix $O$. We claim that,
		\begin{lem}\label{lem:bound_O}
			Denote $x_{ij}$ to be the random variable denoting the probability of hidden state $i$ emitting output $j$. If $\{x_{ij},i \in [n]\}$ are independent and $\E[x_{ij}]=1/m$ for all $i$ and $j$, then for all outputs $j \in [m]$ and hidden states $i\in [n]$, ${\Big|P[o_{t+1}=j|h_t=i]-1/m\Big|\le \sqrt{6\log n/(dm)}}$.
		\end{lem}
		\begin{proof}
			The result is a simple application of Chernoff bound and a union bound. Without loss of generality, let the set of neighbors of hidden state $i$ be the hidden states $\{1,2,\dotsb,d\}$. ${P[o_{t+1}=j|h_t=i]}$ is given by,
			\begin{align}
			P[o_{t+1}=j|h_t=i] = (1/d)\sum_{k=1}^{d}x_{kj}\nonumber
			\end{align}
			Let $X_{ij}=\sum_{k=1}^{d}x_{kj}$. Note that the $x_{kj}$'s are all independent and bounded in the interval $[0,1]$ with $\E[x_{kj}]=1/m$. Therefore we can apply Chernoff bound to show that,
			\begin{align}
			P\Big[|X_{ij}-d/m|&\ge \sqrt{\frac{6d\log n}{m}}\Big]\le 2\exp(-2\log n)\le 2/n^2\nonumber
			\end{align}
			Therefore, ${\Big|P[o_{t+1}=j|h_t=i]-1/m\Big|\le \sqrt{6\log n/(dm)}}$ with failure probability at most $1/n^2$. Hence by performing a union bound over all hidden states and outputs, with high probability $|P[o_{t+1}=j|h_t=i]-1/m|\le \sqrt{6\log n/(dm)}$ for all all outputs $j \in [m]$ and hidden states $i\in [n]$.
		\end{proof}
		Using Lemma \ref{lem:bound_O}, it follows that 
		$$P[o_{t+1},\dotsb,o_{\tau}|h_t=i] \in \Big[\frac{1}{m^{\tau-t}}\Big(1-\sqrt{\frac{3m\log^3 n}{2d}}\Big),\frac{1}{m^{\tau-t}}\Big(1+\sqrt{\frac{24m\log^3 n}{d}}\Big)\Big]$$ This is because using Lemma \ref{lem:bound_O}, conditioned on any hidden state at any time $t$, $${P[o_{t+1}=j|h_t=i]} \in {\Big[\Big(\frac{1}{m}-\sqrt{\frac{6\log n}{dm}}\Big),\Big(\frac{1}{m}+\sqrt{\frac{6\log n}{md}}\Big)\Big]}$$ for all outputs $j$, hidden states $i$  and $t\in[0,\tau]$ hence the probability of emitting the sequence of outputs $\{o_t,\dotsb,o_{\tau}\}$ starting from any hidden state is at most $$ \Big(\frac{1}{m}+\sqrt{\frac{6\log n}{md}}\Big)^{\tau-t}\le \frac{1}{m^{\tau-t}}\Big(1+\sqrt{\frac{24m\tau^2\log n}{d}}\Big)\le \frac{1}{m^{\tau-t}}\Big(1+\sqrt{\frac{24m\log^3 n}{d}}\Big)$$ and similarly for the lower bound. Therefore $O^{(t)}_{i,i}=P[o_{t+1},\dotsb,o_{\tau}|h_{t+1}=i]$ can be bounded as follows
		\begin{align}
		P[o_{t+1},\dotsb,o_{\tau}|h_{t+1}=i] &=P[o_{t+1}|h_{t+1}=i]P[o_{t+2},\dotsb,o_{\tau}|h_{t+1}=i]\le P[o_{t+2},\dotsb,o_{\tau}|h_{t+1}=i]\nonumber\\
		\implies O^{(t)}_{i,i} &\le \frac{1}{m^{\tau-t-1}}\Big(1+\sqrt{\frac{24m\log^3 n}{d}}\Big)\; \forall i,t\nonumber
		\end{align}
		We will now bound the entries of $\tilde{E}^{(t)}$. Note that
		\begin{align}
		1/E^{(t)}_{i,i}&= P[o_{t+1},\dotsb,o_{\tau}|h_{t}=i]\nonumber\\
		\implies & \frac{1}{m^{\tau-t}}\Big(1-\sqrt{\frac{3m\log^3 n}{2d}}\Big)\le 1/E^{(t)}_{i,i}\le \frac{1}{m^{\tau-t}}\Big(1+\sqrt{\frac{24m\log^3 n}{d}}\Big)  \nonumber\\
		\implies& m^{\tau-t}\Big(1-\sqrt{\frac{24m\log^3 n}{2d}}\Big)\le E^{(t)}_{i,i}\le m^{\tau-t}\Big(1+\sqrt{\frac{12m\log^3 n}{d}}\Big)  \nonumber
		\end{align}

		
		We can cancel the factor of $m^{\tau-t-1}$ appearing in the numerator of the upper bound for $E_{ii}^{(t)}$ and denominator of the upper bound for $P[o_{t+1},\dotsb,o_{\tau}|h_{t+1}=i]$ by multiplying $O^{(t)}$ by $m^{\tau-t}$ and dividing $E^{(t)}$ by $m^{\tau-t}$. Let the normalized matrices be $\tilde{O}^{(t)}$ and $\tilde{E}^{(t)}$. Therefore,
		\begin{align}
		&\tilde{O}^{(t)}_{i,i} \le \Big(1+\sqrt{\frac{24m\log^3 n}{d}}\Big)\nonumber\\
		m\Big(1-\sqrt{\frac{24m\log^3 n}{2d}}\Big)\le &\tilde{E}^{(t)}_{i,i}\le m\Big(1+\sqrt{\frac{12m\log^3 n}{d}}\Big)  \nonumber
		\end{align}
		
		Consider any $x$ such that $1^{T}x=0$ and $\norm{x}=1$. Let $\tilde{E}^{(t)}x = \alpha v + x_2$, where $1^Tx_2=0, \norm{x_2}\le 1$ and $v$ is the all ones vector normalized to have unit $\ell_2$-norm. We claim that $\alpha = v^Tx \le \sqrt{\frac{100m^3\log^3 n}{2d}}$. As $|\tilde{E}_{i,i}-\tilde{E}_{j,j}| \le \sqrt{\frac{100m^3\log^3 n}{2d}}$ for all $i\ne j$ and $1^Tx=0$, therefore
		\begin{align}
		|v^T\tilde{E}^{(t)}x| &\le \frac{\max_{i,j}|\tilde{E}_{i,i}-\tilde{E}_{j,j}|}{\sqrt{n}}\absnorm{x}\nonumber\\
		&\le \sqrt{\frac{100m^3\log^3 n}{2dn}}\absnorm{x}\le \sqrt{\frac{100m^3\log^3 n}{2d}}\norm{x}\le \sqrt{\frac{100m^3\log^3 n}{2d}}\nonumber
		\end{align}
		Therefore, $T\tilde{E}^{(t)}x=\alpha v+Tx_2$. Let the second eigenvalue of $T$ be upper bounded by $\lambda$. As $1^Tx_2=0$ therefore $\norm{Tx_2}\le \lambda$. Hence $\norm{T\tilde{E}^{(t)}x}\le \alpha + \lambda$ for any $x$ with $1^Tx=0$ and $\norm{x}=1$. Note that the operator norm of the matrix $\tilde{O}^{(t)}$ is at most $2$ because $\tilde{O}^{(t)}$ is a diagonal matrix with each entry bounded by $\Big(1+\sqrt{\frac{24m\log^3 n}{d}}\Big)\le 2$. Therefore $\norm{\tilde{O}T\tilde{E}^{(t)}x} \le 2(\alpha+\lambda)$ for any $x$ with $1^Tx=0$ and $\norm{x}=1$.
	\end{proof}
\end{proof}
\section{Additional proofs for Section \ref{sec:learnability}: Learnability results}\label{sec:learning_proof}

\noindent \textbf{Assumptions for learning HMMs efficiently:}

For some fixed constants $c_1, c_2, c_3 > 1 $, the HMM should satisfy the following properties for some $c>0$:
\begin{enumerate}
	\item \emph{Transition matrix is well-conditioned:} Both $T$ and the transition matrix $T'$ of the time reversed Markov Chain are well-conditioned in the $\ell_1$-norm: $\sigma_{\min}^{(1)}(T),\sigma_{\min}^{(1)}(T') \ge 1/m^{ c/c_1}$
	\item \emph{Transition matrix does not have short cycles:} For both $T$ and $T'$, every state visits at least $10 \log_m n$ states in $15 \log_m n$ time except with probability $\delta_1 \le 1/n^c$.
	\item \emph{All hidden states have small ``degree'':} There exists $\delta_2$ such that for every hidden state $i$, the transition distributions $T_i$ and $T'_i$ have cumulative mass at most $\delta_2$ on all but $d$ states, with $d \le m^{1/c_2}$ and $\delta_2 \le 1/n^c$. Hence this is a soft ``degree'' requirement.
	\item \emph{Output distributions are random and have small support:} There exists $\delta_3$ such that for every hidden state $i$ the output distribution $O_i$ has cumulative mass at most $\delta_3$ on all but $k$ outputs, with $k\le m^{1/c_3}$ and $\delta_3 \le 1/n^c$. Also, the output distribution $O_i$ is randomly chosen from the simplex on these $k$ outputs.
\end{enumerate}

\learn*

\begin{proof}
	Let the window length $N=2\window+1$ where $\window = 15\log_m n$. We will prove the theorem for $c_1 = 20, c_2 = 16$ and $c_3 = 10$, these can be modified for different tradeoffs. By Lemma \ref{lem:jennrich} from \citet{bhaskara2014smoothed}, the simultaneous diagonalization procedure in Algorithm \ref{alg:learn_tensor} needs $\poly{n,1/\beta,\kappa,1/\epsilon}$ samples to ensure that the decomposition is accurate to an additive error $\epsilon$, with $\beta$ and $\kappa$ defined below.
	
	\begin{lem}\cite{bhaskara2014smoothed}\label{lem:jennrich}
		Suppose we are given a tensor $M+E\in \mathbb{R}^{m \times n \times p}$ with the entries of $E$ being bounded by $\epsilon=\poly{1/\kappa, 1/n, 1/\beta}$ and $M$ has a decomposition $M = \sum_{i=1}^{R}A_i \otimes B_i \otimes C_i$ which satisfies-
		\begin{enumerate}
			\item The condition numbers $\kappa(A), \kappa(B) \le \kappa$. 
			\item The column vectors of $C$ are not close to parallel: for all $i\ne j$, ${\largenorm{{\frac{w_i}{\norm{w_i}}-\frac{w_j}{\norm{w_j}}}} \ge \beta}$
			\item The decompositions are bounded: for all $i$, $\norm{u_i}, \norm{v_i}, \norm{w_i}\le K$
		\end{enumerate}
		then the simultaneous decomposition algorithm recovers each rank one term in the decomposition of $M$ (up to renaming), within an additive error of $\epsilon$.
	\end{lem}
	 
	Lemma \ref{lem:bound_dist} (proved in Section \ref{sec:lemma2}) shows that the observation matrix $O$ satisfies condition 2 in Lemma \ref{lem:jennrich} with $\beta = 1/n^{6.5}$.
	
	\begin{restatable}{lem}{bounddist}\label{lem:bound_dist}
		If each column of the observation matrix is uniformly random on a support of size $k$, then ${\norm{{\frac{O_i}{\norm{O_i}}-\frac{O_j}{\norm{O_j}}}} \ge 1/n^{6.5}}$ with high probability over the choice of $O$.
	\end{restatable}

	Note that as $A,B$ and $O$ are all stochastic matrices, all the factors have $\ell_2$ norm at most 1, therefore condition 3 in Lemma \ref{lem:jennrich} is satisfied with $K=1$. Hence if we show that the condition numbers $\kappa(A), \kappa(B)\le \poly{n}$, then each rank one term $A_i \otimes B_i \otimes C_i$ can be recovered up to an additive error of $\epsilon$ with $\poly{n,1/\epsilon}$ samples. As $\pi_i$ is at least $1/\poly{n}$ for all $i$, this implies that $A, B, C$ can be recovered with additive error $\epsilon$ with $\poly{n,1/\epsilon}$ samples. As $O$ can be recovered from $C$ by normalizing $C$ to have unit $\ell_1$ norm, hence $O$ can be recovered up to an additive error $\epsilon$ with $\poly{n, 1/\epsilon}$ samples. 
	
	If the estimate $\hat{A}$ of $A$ is accurate up to an additive error $\epsilon$, then the estimate $\hat{A}^{(\window-1)}$ of $A^{(\window-1)}$ is accurate up to an additive error $O(m\epsilon)$. Therefore, if the estimate $\hat{O}$ of $O$ is also accurate up to an additive error $\epsilon$ then the estimate of $\hat{A}'=(\hat{O}\odot \hat{A}^{(\window-1)})$ of $A'$ is accurate up to an additive error $O(n\epsilon)$. Further, if the matrix $A'$ also has condition number at most $\poly{n}$, then the transition matrix $T$ can be recovered using Algorithm \ref{alg:learn_tensor} with up to an additive error of $\epsilon$ with $\poly{n,1/\epsilon}$ samples. Hence we will now show that the condition number $\kappa(A)\le \poly{n}$. The proof for an upper bound for $\kappa(A')$ and $\kappa(B)$ follows by the same argument.
	
	Define the $(1-\delta)$-support of any distribution $p$ as the set such that $p$ has mass at most $\delta$ outside that set. For convenience, define $\delta=\max\{\delta_1, \delta_2, \delta_3\}$ in the conditions for Theorem \ref{thm:learnability}. We will find the probability that the Markov chains only undertakes transitions belonging to the  $(1-\delta)$-support of the current hidden state at each time step. As the probability of transitioning to any state outside the $(1-\delta)$-support of the current hidden state is at most $\delta$, and the transitions at each time step are independent conditioned on the hidden state, the probability that the Markov chains only undertakes transitions belonging to the $(1-\delta)$-support of all hidden state for each of the $n$ time steps is at least $(1-\delta)^{\window} \ge 1-{2\window}{\delta}$. \footnote{For $x \in [0,0.5], n>0, xn \le 1, (1-x)^n \in [1-2xn, 1-xn/2]$.}
	
	By the same argument, the probability of a sequence of hidden states always emitting an output which belongs to the $(1-\delta)$-support of the output distribution of the hidden state at that time step is at least $(1-\delta)^{\window}\ge 1-{2\window}{\delta}$.
	
	We now show that two sequences of hidden states which do not intersect have large distance between their output distributions. The output alphabet over a window of size $\tau$ has size $K=m^\window$. Let $\{a_i, i \in [K]\}$ be the set of all possible output in $\window$ time steps. For any sequence $s$ of hidden states over a time interval $\window$, define $o_{s}$ to be the vector of probabilities of output strings conditioned on any sequence of hidden states $s$, hence $o_{s} \in \mathbb{R}^{K}$ and the first entry of $o_s$ equals $P(a_1|s)$. Lemma \ref{determine} (proved in Section \ref{sec:lemma3}) shows that the output distributions of sample paths which do not meet in $\tau$ time steps is large with high probability.
	
	\begin{restatable}{lem}{determine}\label{determine}
		Let $s_i$ and $s_j$ be two sequences of $\window$ hidden states which do not intersect. Also assume that $s_i$ and $s_j$ have the property that the output distribution at every time step corresponds to the $(1-\delta)$ support of the hidden state at that time step. Let $o_{s_i}$ be the vector of probabilities of output strings conditioned on any sequence of hidden states $s_i$. Also, assume that $s_i$ and $s_j$ both visit at least $(1-\alpha) n$ different hidden states. Then, 
		\begin{align}
		P\Big[\absnorm{o_{s_i}-o_{s_j}} = 1\Big]\ge \Big(\frac{4m^2}{d}\Big)^{0.5(1-\alpha) \window}\nonumber
		\end{align}
	\end{restatable}

	Note that $\alpha \le 1/3$ according to our condition 3 of Theorem \ref{thm:learnability}. Also, as $k<m^{1/10}$, therefore $\Big(\frac{4k^2}{m}\Big)^{ \window/3}\le (4/m^{4/5})^{\window/3}$. Now consider the set $\mathcal{M}$ of all sequences with the property that the transition at every time step corresponds to the $(1-\delta)$-support of the hidden state at that time step. As the $(1-\delta)$-support of every hidden state has size at most $d$, $|\mathcal{M}|\le nd^{\window}$. Now for a sequence $s_i$, consider the set $\mathcal{M}_{s_i}$ of all sequences $s_j$ which do not intersect that sequence. By a union bound---
	\begin{align}
	P\Big[\absnorm{o_{s_i}-o_{s_j}} = 1 \; \forall s_j \in \mathcal{M}_{s_i}\Big] \ge 1- nd^{\window}(4/m^{4/5})^{\window/3} \nonumber
	\end{align}
	Now doing a union bound over all sequences $s_i$,	
	\begin{align}
	P\Big[\absnorm{o_{s_i}-o_{s_j}} = 1 \; \forall s_i; s_j \in \mathcal{M}_{s_i}\Big] &\ge 1- n^2d^{2\window}(4/m^{4/5})^{\window/3}\nonumber\\
	&\ge 1- \frac{n^2m^{30\log_m n/16}4^{5\log_m n}}{m^{4\log_m n}}\nonumber\\
	&\ge 1- \frac{n^{5/\log_4 m}n^{31/8}}{n^4}\nonumber
	\end{align}
	which is $1-o(1)$. Hence every non-intersecting sequence of states in $\mathcal{M}$ has a different emission with high probability over the random assignment.
	
	
	Using this property of $O$, we will bound the condition number of $A$. We will lower bound $\sigma_{\min}^{(1)}(A)$. As $\sigma_{\min}^{(2)}(A) \ge\sigma_{\min}^{(1)}(A)/ \sqrt{n} $ and $\sigma_{\max}^{(2)}(A) \le \sqrt{n}$, $\kappa(A) \le n \sigma_{\min}^{(1)}(A)$. 
	
	Consider any $x\in\Real^n$ such that $\absnorm{x}=2$. We aim to show that $\absnorm{Ax}$ is large. Let $x^+$ be the vector of all positive entries of $x$ i.e. $x^+_i = \max(x_i,0)$ where $x_i$ denotes the $i$th entry of vector $x$. Similarly, $x^-$ is the vector of all negative entries of $x$, $x^-_i = \max(-x_i,0)$. We will find a lower bound for $\absnorm{Ax^+ -Ax^-}$. Note that because $A$ is a stochastic matrix, there exists and $x$ that minimizes $\absnorm{Ax}$ and has $1^Tx = 0$. Hence we will assume without loss of generality that $1^Tx = 0$, and hence $x^+$ and $x^-$ are valid probability distributions. Hence our goal is to show that for any two initial distributions $x^+$ and $x^-$ which have disjoint support,  $\absnorm{Ax^+ -Ax^-}\ge 1/\poly{n}$ which implies $\absnorm{Ax^+ -Ax^-}\ge 1/\poly{n}$ for any $x\in\Real^n$ such that $\absnorm{x}=2$. 
	
	Note that $\absnorm{T^n x} >(\sigma_{\min}^{(1)}(T))^{\window}$. Hence the distributions $x^+$ and $x^-$ do not couple in $n$ time steps with probability at least $(\sigma_{\min}^{(1)}(T))^{\window}$. Note that $A{x}^+$ is a vector where the $i$th entry is the probability of observing string $a_i$ with the initial distribution ${x}^+$. Let $\mathcal{U}$ denote the set of all $n^{\window}$ sequences of hidden states over $\window$ time steps. We exclude all sequences which do not visit $(1-\alpha)\window$ hidden states in $\window$ time steps and sequences where at least one transition is outside the $(1-\delta)$ support of the current hidden state. Let $\mathcal{S}$ be the set of all sequences of hidden states which visit at least $(1-\alpha) n$ hidden states and where each transition is in the  $(1-\delta)$ support of the current hidden state. Let $\mathcal{X}=\mathcal{U}\backslash \mathcal{S}$. Note that for any distribution $p$, $Ap$ can be written as $Ap = \sum_{s\in\mathcal{X}}P(s|p) o_{s}$ where $P(s|p)$ is the probability of sequence $s$ with initial distribution $p$. Recall that $o_s \in \mathbb{R}^K$ is the vector of probabilities of outputs over the $K=m^{\window}$ size alphabet conditioned on the sequence of hidden states $s$. Taking $p$ to be $e_i$, each column $c_i$ of $A$ can be expressed as $c_i = \sum_{s\in\mathcal{X}}P(s|e_i) o_{s}$. Restricting to sequences in $\mathcal{S}$, we define two new matrices $A_{1}$ and $A_{2}$, with the $i$th column $c_{i,1}$ of $A_{1}$ defined as $c_{i,1} = \sum_{s\in\mathcal{S}}P(s|e_i) o_{s}$ and the $i$th column $c_{i,2}$ of $A_{2}$ is analogously defined as $c_{i,2} = \sum_{s\in\mathcal{X}}P(s|e_i) o_{s}$. Note that $A= A_{1}+ A_{2}$. Also, as every hidden state visits less than $(1-\alpha)\window$ hidden states over $\window$ time steps with probability at most $\delta$, and the probability of taking at least one transition outside the $(1-\delta)$ support is at most $\window \delta$, therefore $\sum_{s\in\mathcal{X}}P(s|e_i) \le \delta+ \window \delta$ for all $i$. Hence $\absnorm{A_{2}p} \le  \delta+ \window \delta$ for any vector $p$ with $\absnorm{p}=1$. 
	
	
	We will now lower bound $\absnorm{A_1 x}$. The high level proof idea is that if the two initial distributions do not couple then some sequences of states have more mass in one of the two distributions. Then, we use the fact that most sequences of states comprise of many distinct states to argue that two sequences of hidden states which do not intersect lead to very different output distributions. We combine these two observations to get the final result.
	
	Consider any $O$ which satisfies the property that every non-intersecting sequence of states has a different emission. We divide the output distribution $o_{s}$ of any sequence $s$ into two parts, $o_{s}^{(1)}$ and $o_{s}^{(2)}$. Define the $(1-\delta)$-support of $o_{s}$ as the set of possible outputs obtained when the emission at each time step belongs to the $(1-\delta)$ support of the output at that time step. We define $o_{s}^{(1)}$ as $o_{s}$ restricted to the $(1-\delta)$ support of $o_{s}$ and $o_{s}^{(2)}$ to be the residual vectors such that $o_{s} = o_{s}^{(1)} + o_{s}^{(2)}$. Note that $\absnorm{o_{s}^{(2)}}$ is at most $\window \delta$ for any sequence $s_i$. 
	
	Using our decomposition of the output distributions $o_{s}$ we will decompose $A_1$ as the sum of two matrices $A_1'$ and $A_1''$. The $i$th column $c_i'$ of $A_1'$ is defined as $c_i' = \sum_{s\in\mathcal{S}}P(e_i|s) o_{s}^{(1)}$. Similarly, the $i$th column $c_i''$ of $A_1''$ is defined as $c_i'' = \sum_{s\in\mathcal{S}}P(e_i|s) o_{s}^{(2)}$. As $\absnorm{o_{s_i}^{(2)}}\le \window \delta$ for every sequence $s_i$, $\absnorm{A_1''p}\le \window \delta$ for any vector $p$ with $\absnorm{p}=1$. 
	
	We will further divide every sequence $s$ into a set of augmented sequences $\{s_{a_1}, s_{a_2}, \dotsb, s_{a_K}\}$. Recall that $K=d^{\window}$ is the size of the output space in $\window$ time steps and $\{a_i, i \in [K]\}$ is the set of all possible output in $\window$ time steps. The probability of sequence $s_{a_i}$ is the product of the probability of sequence $s$ times the probability of the sequence $s$ emitting the observation $a_i$. Hence, the probability of each augmented sequence equals the product of the probability of making the corresponding transition at each time step and emitting the corresponding observation at that time step.
	
	For each output string $a_i$, there is a set of augmented sequences which have non-zero probability of emitting $a_i$. Consider the first string $a_1$. Let $\mathcal{S}_1^+$ be the set of all augmented sequences from the $x^+$ distribution with non-zero probability of emitting $a_1$, similarly let $\mathcal{S}_1^-$ be the set of all augmented sequences from the $x^-$ distribution with non-zero probability of emitting $a_1$. For any set of augmented sequences $\mathcal{S}$, let $|\mathcal{S}|$ denote the total probability mass of augmented sequences in $\mathcal{S}$. 
	
	
	We now show that any assignment of augmented sequences to outputs $a_i$ induces a coupling for two Markov chains $m^+$ and $m^-$ starting with initial distributions $x^+$ and $x^-$ respectively and having transition matrix $T$. We denote two sequences $s^+$ and $s^-$ as having coupled if they meet at some time step $u$ and traverse the Markov chain together after $u$, and the probability of $s^+$ equals the probability of $s^-$. Let $\mathcal{C}$ denote some coupling, and $\bar{\mathcal{C}}$ denote the total probability mass on sequences which have not been coupled under $\mathcal{C}$. Note that the total variational distance between the distribution of hidden states at time step $n$ from the starting distribution $x^+$ and $x^-$ satisfies
	\begin{align}
	\absnorm{T^\window x^+-T^\window x^-}\le {\Big|\bar{\mathcal{C}}\Big|}\nonumber
	\end{align}
	for any coupling $\mathcal{C}$ with uncoupled mass $\bar{\mathcal{C}}$. This follows because all coupled sequences have the same distribution at time $\window$, hence the distance between the distributions is at most the mass on the uncoupled sequences.
	
	We claim that the sets of augmented sequences $\mathcal{S}_1^+$ and $\mathcal{S}_1^-$ can be coupled with residual mass ${|(A_{1}'x^+)_{1} -(A_{1}'x^-)_{1}|}$, where $(A_{1}'x^+)_{1}$ denotes the first entry of $A_{1}'x^+$. To verify this, first assume without loss of generality that $(A_{1}'x^+)_{1}>(A_{1}'x^-)_{1}$.  Let $P(a_1,h_i|x^+)$ be the probability of outputting $a_1$ in $\window$ time steps and being at hidden state $h_i$ at time 0, given the initial distribution $x^+$ at time 0. This is the sum of the probability of augmented sequences in $\mathcal{S}_1^+$ which start from hidden state $h_i$. Any coupling of the augmented sequences also induces a coupling of the probability masses $p^+=\{P(a_1,h_i|x^+), i \in [n]\}$ and $p^-=\{P(a_1,h_i|x^-), i \in [n]\}$. We will show that all of the probability mass $p^-$ can be coupled. Consider a simple greedy coupling scheme $\mathcal{C}_1$ which picks a starting state $i$ in $p^-$, traverses along the transitions from the state $i$ and couples as much probability mass on sequences starting from $i$ as possible whenever it meets a sequence from $p^+$, and repeats for all starting states, till there are no more sequences which can be coupled. We claim that the algorithm terminates when all of the probability mass in $p^-$ has been coupled. We prove by contradiction. Assume that there exists some probability mass $p^-$ which has not been coupled in the end. There must be also be some probability mass $p^+$ which has not been coupled. But all augmented sequences starting from hidden state $i$ meet with all sequences from hidden state $j$, so this means that more probability mass from $p^-$ can be coupled to $p^+$. This contradicts the assumption that there are no more augmented sequences which can be coupled. Hence all of the probability mass in $p^-$ has been coupled when the greedy algorithm terminates. Hence coupling $\mathcal{C}_1$ has residual mass $\Big|\bar{\mathcal{C}_1}\Big|$ at most ${|(A_{1}'x^+)_{1} -(A_{1}'x^-)_{1}|}$.
	
	Now, consider all outputs $a_i$ and couplings $\mathcal{C}_i$. Let $\mathcal{C}=\cup_i \mathcal{C}_i$. As our argument only couples the augmented sequences the mass $\absnorm{A_1''x^+}+ \absnorm{A_1''x^-}+\absnorm{A_2x^+}+\absnorm{A_2x^-}$ is never uncoupled. The total uncoupled mass is,
	\begin{align}
	\Big|\bar{\mathcal{C}}\Big| &= \sum_{i}^{}\Big|\bar{\mathcal{C}}\Big| + \absnorm{A_1''x^+}+ \absnorm{A_1''x^-}+\absnorm{A_2x^+}+\absnorm{A_2x^-}\nonumber\\
	\implies\Big|\bar{\mathcal{C}}\Big| &\le \absnorm{A_{1}'x^+ -A_{1}'x^-}+6\window\delta\nonumber\\
	\implies \absnorm{T^\window x^+-T^\window x^-}&\le \absnorm{A_{1}'x^+ -A_{1}'x^-}+6\window \delta\nonumber
	\end{align}
	Note that $\absnorm{T^\window x^+-T^\window x^-}\ge (\sigma_{\min}^{(1)}(T))^{\window} $. By condition 2 in Theorem \ref{thm:learnability}, $\sigma_{\min}^{(1)}(T)\ge 1/m^{c/20}$ therefore $\absnorm{T^\window x^+-T^\window x^-}\ge 1/n^{3c/4}$. Note that $\absnorm{Ax} \ge \absnorm{A_1'x}-\absnorm{A_1''x}-\absnorm{A_2 x}\ge 1/n^{3c/4}+10\window\delta$. As $\delta \le 1/n^c$, therefore $\absnorm{Ax}\ge 1/n^{3c/4}-0.5/n^{3c/4} \ge 0.5/n^{3c/4}$.
\end{proof}

We also mention the following corollary of Theorem \ref{thm:learnability}. Corollary \ref{thm:learnability_l2} is defined in terms of the minimum singular value of the matrix, $\sigma_{\min}^{(2)}(T)$, and is a slightly weaker but more interpretable version of Theorem \ref{thm:learnability}. The conditions are the same as in Theorem \ref{thm:learnability} with different bounds on $\delta_1, \delta_2, \delta_3$.
\begin{cor}\label{thm:learnability_l2}
	If an HMM satisfies $\delta_1,\delta_2,\delta_3\le 1/n^2$ and $\sigma_{\min}^{(2)}(T),\sigma_{\min}^{(2)}(T')\ge 1/m^{1/20}$ then with high probability over the choice of $O$, the parameters of the HMM learnable to within additive error $\epsilon$ with observations over windows of length $2\tau +1, \tau = 15\log_m n$ with the sample complexity being $\poly{n,1/\epsilon}$.
\end{cor}

On a side note, though $\sigma_{\min}^{(2)}(T)$ is much more easier to interpret than $\sigma_{\min}^{(1)}(T)$ -- for example, if the transition matrix is symmetric then the singular values are the same as the eigenvalues, and the eigenvalues of a matrix have well-known connections to the properties of the underlying graph, but, because $T$ is a stochastic matrix and all columns have unit $\ell_1$ norm, the $\ell_1$ norm seems better suited to measuring the gain of the matrix. For example, if the transition matrix has a single state which transitions to $d$ states with equal probability, then $\sigma_{\min}^{(2)}(T)\le 1/\sqrt{d}$ by choosing the vector which has unit mass on that state, hence if $d$ is large then $\sigma_{\min}^{(2)}(T)$ is always small even when the transition matrix is otherwise well-behaved. 

\subsection{Proof of Lemma \ref{lem:bound_dist}}\label{sec:lemma2}

\bounddist*
\begin{proof}
	The proof is in two parts. In the first part we show that $\absnorm{O_i-O_j}\ge 1/n^5$ with high probability. In the second part we show that ${\largenorm{{\frac{O_i}{\norm{O_i}}-\frac{O_j}{\norm{O_j}}}} \ge 1/n^{6.5}}$ if $\absnorm{O_i-O_j}\ge 1/n^5$.
	
	We will show that $\absnorm{O_i-O_j}\ge 1/n^5$ with high probability in the smoothed sense, which implies that it is also true in our model where they are chosen uniformly on a small support. Consider any two distributions $O_i$ and $O_j$. Let $O_i$ have largest mass on some state $f_i$ and next largest mass on another state $g_i$. Similarly, let $O_j$ have largest mass on some state $f_j$ and next largest mass on another state $g_j$. Let $x_1$ and $x_2$ be random variables uniformly distributed in $[0,1/n^2]$. Say we perturb $O_i$ by subtracting $x_1$ from the probability of $f_i$ and adding $x_1$ to the probability of $g_i$. Similarly, say we perturb $O_j$ by subtracting $f_2$ from the probability of $u_j$ and adding $x_2$ to the probability of $g_j$. With probability $1/n^3$ over the choice of $x_1$ and $x_2$, $|x_1-x_2|\ge 1/n^5$, which implies $\absnorm{O_i-O_j}\ge 1/n^5$. Therefore by a union bound over all pairs $O_i$ and $O_j$, $\absnorm{O_i-O_j}\ge 1/n^5$ with high probability. 
	
	We now show that ${\largenorm{{\frac{O_i}{\norm{O_i}}-\frac{O_j}{\norm{O_j}}}} \ge 1/n^{6.5}}$ when $\absnorm{O_i-O_j}\ge 1/n^5$. We prove the contrapositive via the following Lemma.
	
	\begin{lem}\label{lem:l1_l2}
		For any two vectors $v_1$ and $v_2$, ${\largeabsnorm{{\frac{v_1}{\absnorm{v_1}}-\frac{v_2}{\absnorm{v_2}}}} < 1/n^5}$ if ${\largenorm{{\frac{v_1}{\norm{v_1}}-\frac{v_2}{\norm{v_2}}}} \le 1/n^{6.5}}$.
	\end{lem}
	\begin{proof}
		As the claim is scale invariant, assume $\norm{v_1}=1$ and $\norm{v_2}=1$. As $\norm{v_1-v_2}\le 1/n^{6.5}$, therefore $\absnorm{v_1-v_2}\le 1/n^{6}$. Therefore $|\absnorm{v_1}-\absnorm{v_2}|\le 1/n^{6}$. Let $\absnorm{v_1}=x$, where $x\ge 1$, $\absnorm{v_2}=x+\delta$, where $|\delta|\le 1/n^{6}$. 
		\begin{align}
		\largenorm{{\frac{v_1}{\absnorm{v_1}}-\frac{v_2}{\absnorm{v_2}}}} &=
		\largenorm{{\frac{v_1}{x}-\frac{v_2}{x+\delta}}}\nonumber\\
		&= \frac{1}{x}\largenorm{{{v_1}-\frac{v_2}{1+\delta/x}}}\nonumber	\\
		&= \frac{1}{x}\norm{{{v_1}-{v_2}(1+\epsilon/x)}}\nonumber
		\end{align}
		for some $\epsilon$ with $|\epsilon| \le 2|\delta|$. Therefore using the triangle inequality,
		\begin{align}
		\largenorm{{\frac{v_1}{\absnorm{v_1}}-\frac{v_2}{\absnorm{v_2}}}} &\le 
		\frac{1}{x}\norm{{v_1}-{v_2}}+\frac{\epsilon}{x}\norm{v_2}\nonumber\\
		&\le 1/n^{6}+2/n^{6}< 1/n^{5.5}\nonumber\\
		\implies \largeabsnorm{{\frac{v_1}{\absnorm{v_1}}-\frac{v_2}{\absnorm{v_2}}}} &< 1/n^{5}\nonumber
		\end{align}
	\end{proof}
	Using Lemma \ref{lem:l1_l2}, it follows that ${\largenorm{{\frac{O_i}{\norm{O_i}}-\frac{O_j}{\norm{O_j}}}} > 1/n^{6.5}}$ when $\absnorm{O_i-O_j}\ge 1/n^5$

\end{proof}

\subsection{Proof of Lemma \ref{determine}}\label{sec:lemma3}

\determine*
	\begin{proof}
		For the output distributions corresponding to sequence $s_i$ and $s_j$ to not have disjoint supports, the hidden state visited by the two sequences at every time step must have overlapping output distributions. The probability of any pair of hidden states having overlapping support is at most $1-(1-2k/m)^k\le 4k^2/m$. 
		
		Consider the graph $G$ on $n$ nodes, where we connect node $u$ and node $v$ if sequence $s_i$ and $s_j$ are simultaneously at hidden states $v$ and $v$ at some time step. As each sequence visits at least $(1-\alpha)\window$ different hidden states, at least $(1-\alpha)\window$ nodes in $G$ have non-zero degree. Consider any connected component $C$ of the graph with $p$ nodes. Note that the probability of the output distributions corresponding to each edge in the connected component $C$ to be overlapping equals the probability of the output distributions of each node in $C$ to be overlapping. This is at most $\Big(\frac{4k^2}{m}\Big)^{p-1}$. Now, let there be $M$ connected components in  graph each of which has $p_i$ nodes. The probability of the sequences having the same support is at most
		\begin{align}
		P\Big[\absnorm{o_{s_i}-o_{s_j}} < 1\Big]&\le \Pi_{i=1}^M \Big(\frac{4k^2}{m}\Big)^{p_i-1}\nonumber\\
		&\le \Big(\frac{4k^2}{m}\Big)^{\sum_{i}^{}(p_i-1)} 
		\le \Big(\frac{4k^2}{m}\Big)^{(1-\alpha) \window - M} \nonumber\\
		&\le \Big(\frac{4k^2}{m}\Big)^{(1-\alpha) \window/2}\nonumber
		\end{align}
		where the last step follows because $\sum_{i}^{}p_i$ is the number of nodes which have non-zero degree, which is at least $ \ge (1-\alpha)t$, and as every connected component has at least 2 nodes, there are at most $(1-\alpha)t/2$ connected components, therefore $M\le (1-\alpha)t/2$.
	\end{proof} 

%
\section{Additional proofs for Section \ref{sec:identifiability}: Identifiability results}

\identifiability*
\begin{proof}
	We first show that the matrices $A$, $B$ and $A'=(O \odot A^{(\tau-1)})$ become full rank using observations over a window of length $N=2\tau +1$, where $\tau = 2 \lceil \log_m n \rceil$.
	
	Let's fix the window length to be $N=2\tau +1$, where $\tau = 2 \lceil \log_m n \rceil$.  We choose the transition matrix $T$ to be the permutation which is supported by $\mathcal{S}$. By the requirement in Theorem \ref{thm:identifiability}, $T$ is composed of cycles all of which have at least $\tau$ hidden states. Without loss of generality, assume that all cycles in $T$ are composed of sequences of hidden states $\{h_{i},h_{i+1},\dotsb,h_{j-1},h_{j}\}$ for $i,j \in [n], i\le j$. To further simplify notation we also assume, without loss of generality, that each cycle of $T$ is a cyclic shift of the hidden states $\{h_{i},h_{i+1},\dotsb,h_{j-1},h_{j}\}$.
	
	To show that $A,B$ and $A'$ are full rank, it is sufficient to show that they become full rank for some particular choice of $O$. Say that we choose each hidden state $h_i$ to always deterministically output some character $a_i$. We show that $A$ is full-rank for some assignment of outputs to hidden states. The same argument will also imply that $B$ and $A'$ are full rank. Because the transition matrix is a permutation, Markov chains starting in two different starting states are at different hidden states at every time step. Also, because all cycles in the permutation are longer than $\tau$, at every time step at least one of the two states visited by the two initial starting states has not been visited so far by either of the two initial starting states. Hence, the probability of the emissions corresponding to the two initial starting states being the same at any time step is $(1/m)$. Therefore, the probability that the emissions for two different hidden states across a $\tau$ length windows is the same is $(1/m)^{\lceil 2\log_m n \rceil}\le 1/n^2$. As the total number of pairs of hidden states is $n(n-1)/2$, by a union bound, the probability there exists some observation matrix such that the emissions for the $\lceil 2\log_m n \rceil$ window corresponding to all initial states are different is strictly great than 0. By choosing the $O$ which has this property, the matrix $A$ is full rank. By the same reasoning, $A'$ and $B$ are also full rank.

	Hence, as we have shown that there exists some transition matrix $T$ which is supported by $\mathcal{S}$ and observation matrix $O$ such that $A, A'$ and $B$ become full rank. Hence $A, A'$ and $B$ are full rank except for a measure zero set of $T$ and $O$. As the matrix $C$ is linearly independence except for a measure zero set of $O$, Kruskal's condition is satisfied except for a measure zero set of $T$ and $O$, and $A'$ is full rank, hence due to Algorithm \ref{alg:learn_tensor}, the set of HMMs is identifiable except for a measure zero set of $T$ and $O$. This concludes the proof.
	
\end{proof}

\tightwindow*
\begin{proof}
	The proof idea is similar to Theorem \ref{thm:identifiability}. We will choose some particular choice of $T$ and $O$ such that the matrices $A,B$ and $A'$ become full rank. Consider $T$ to be the permutation on $n$ hidden states which performs a cyclic shift of the hidden states i.e. $T_{ij}=1$ if $j=(i+1)\mod n$ for $0\le i,j\le k-1$ and $T_{ij}$ is 0 otherwise. We show that for this particular choice of $T$, there exists a choice of $O$ such that $A$ becomes full-rank for $t=\lceil \log_m n\rceil +1$. As in the proof of Proposition \ref{thm:identifiability}, we choose each hidden state to deterministically output a character, hence each column of $O$ has one non-zero entry. We show that $A$ is full-rank for some choice of $O$, the same argument will also imply that $B$ and $A'$ are full rank for that $O$.
	
	Because the transition matrix is a cyclic shift, we can reformulate the problem of finding a suitable $O$ as that of finding a length $n$ $m$-ary string all of whose length $\tau=\lceil \log_m n\rceil$ cyclic-substrings are unique (a substring is defined as a continuous subsequence), treating the string as cyclic so that the ends wrap around. De Bruijn sequences have exactly this property. A De Brujin sequence of length $k$ is a cyclic sequence in which every possible length $\lfloor \log_m k \rfloor$ $m$-ary string occurs exactly once as a substring. Furthermore, these sequences also have the property that all substrings longer than $\lfloor \log_m k \rfloor$ are unique \cite{flaxman2004strings}. Hence choosing a De Bruijn sequence of length $k=n$ ensures that all substrings of length $\lceil \log_m n\rceil \ge \lfloor \log_m n\rfloor$ are unique.
	
	Hence we have shown that there exists a choice of $O$ such that the matrix $A$ becomes full rank with windows of length $2\lceil \log_m n\rceil +1$ for some choice of $O$. Hence by the same argument as in the proof of Proposition \ref{thm:identifiability}, all HMMs except those belonging to a measure zero set of $T$ and $O$ become identifiable with windows of length $2\lceil \log_m n\rceil +1$.
\end{proof}

\idenbound*
	\begin{proof}
		There are $m^c$ possible outputs over a window of length $c$. Define a larger alphabet of size $m^c$ which denotes output sequences over a window of length $c$. Dividing the window of length $t$ into $t/c$ segments, the probability of a output sequence only depends on the counts of outputs from the larger alphabet and follows a multinomial distribution. The total number of possible counts is at most $\Big(\frac{2t}{c}\Big)^{m^c}$ as each of the $m^c$ outputs can have a count from 0 to $t/c$. Therefore windows of length $t$ give at most $\Big(\frac{2t}{c}\Big)^{m^c}$ independent measurements, hence the window length has to be at least $O(n^{1/{m^c}})$ for the model to be identifiable as an HMM has $O(n^2+nm)$ independent parameters.
	\end{proof}

\end{document}